\documentclass{article}

\PassOptionsToPackage{numbers, compress}{natbib}

     \usepackage[preprint]{neurips_2025}

\usepackage{hyperref}
\usepackage{amsmath,amssymb}
\usepackage{amsthm}

\newtheorem{theorem}{Theorem}
\newtheorem{lemma}{Lemma}

\newtheorem{assumption}{Assumption}

\usepackage{adjustbox}
\usepackage{booktabs}
\usepackage[utf8]{inputenc} 
\usepackage[T1]{fontenc}    
\usepackage{hyperref}       
\usepackage{url}            
\usepackage{booktabs}       
\usepackage{amsfonts}       
\usepackage{nicefrac}       
\usepackage{microtype}      
\usepackage[table,xcdraw]{xcolor}        
\usepackage{amsmath,amssymb}
\usepackage{notation}
\usepackage{graphicx}
\usepackage{natbib}
\usepackage{algorithm}
\usepackage{algpseudocode}
\usepackage{multirow}
\usepackage{adjustbox}
\usepackage{subcaption}
\usepackage{siunitx}
\usepackage{multirow}
\usepackage{comment}

\definecolor{full}{HTML}{E76BF3}
\definecolor{trun}{HTML}{00BA38}
\definecolor{ours}{HTML}{619CFF}

\usepackage{colortbl}
\usepackage{xcolor}
\definecolor{lightteal}{rgb}{0.7, 0.9, 0.9}
\usepackage{wrapfig}

\title{Denoising Score Distillation: From Noisy Diffusion Pretraining to One-Step High-Quality Generation
}

\author{
    Tianyu Chen$^{*, 1}$, Yasi Zhang$^{*, 2}$,  Zhendong Wang$^3$,  Ying Nian Wu$^2$,  \\   \textbf{Oscar Leong}$^{\dagger,2}$\textbf{,}   \textbf{Mingyuan Zhou}$^{\dagger,1}$ \\
    $^*$Equal contribution  $^\dagger$Equal advising  \\
    $^1$University of Texas at Austin  $^2$University of California, Los Angeles  $^3$Microsoft  
}

\begin{document}
\maketitle

\begin{abstract}

Diffusion models have achieved remarkable success in generating high-resolution, realistic images across diverse natural distributions. However, their performance heavily relies on high-quality training data, making it challenging to learn meaningful distributions from corrupted samples. This limitation restricts their applicability in scientific domains where clean data is scarce or costly to obtain. In this work, we introduce denoising score distillation (DSD), a surprisingly effective and novel approach for training high-quality generative models from low-quality data. DSD first pretrains a diffusion model exclusively on noisy, corrupted samples and then distills it into a one-step generator capable of producing refined, clean outputs. While score distillation is traditionally viewed   as a method to accelerate diffusion models, we show that it can also significantly enhance sample quality, particularly when starting from a degraded teacher model. Across varying noise levels and datasets, DSD consistently improves generative performance—we summarize our empirical evidence in Fig. \ref{fig:first_plot}. Furthermore, we provide theoretical insights showing that, in a linear model setting, DSD identifies the eigenspace of the clean data distribution’s covariance matrix, implicitly regularizing the generator. This perspective reframes score distillation as not only a tool for efficiency but also a mechanism for improving generative models, particularly in low-quality data settings.

\begin{figure}[htbp]
    \centering
    \begin{subfigure}{0.32\textwidth}
        \centering
        \includegraphics[width=\linewidth]{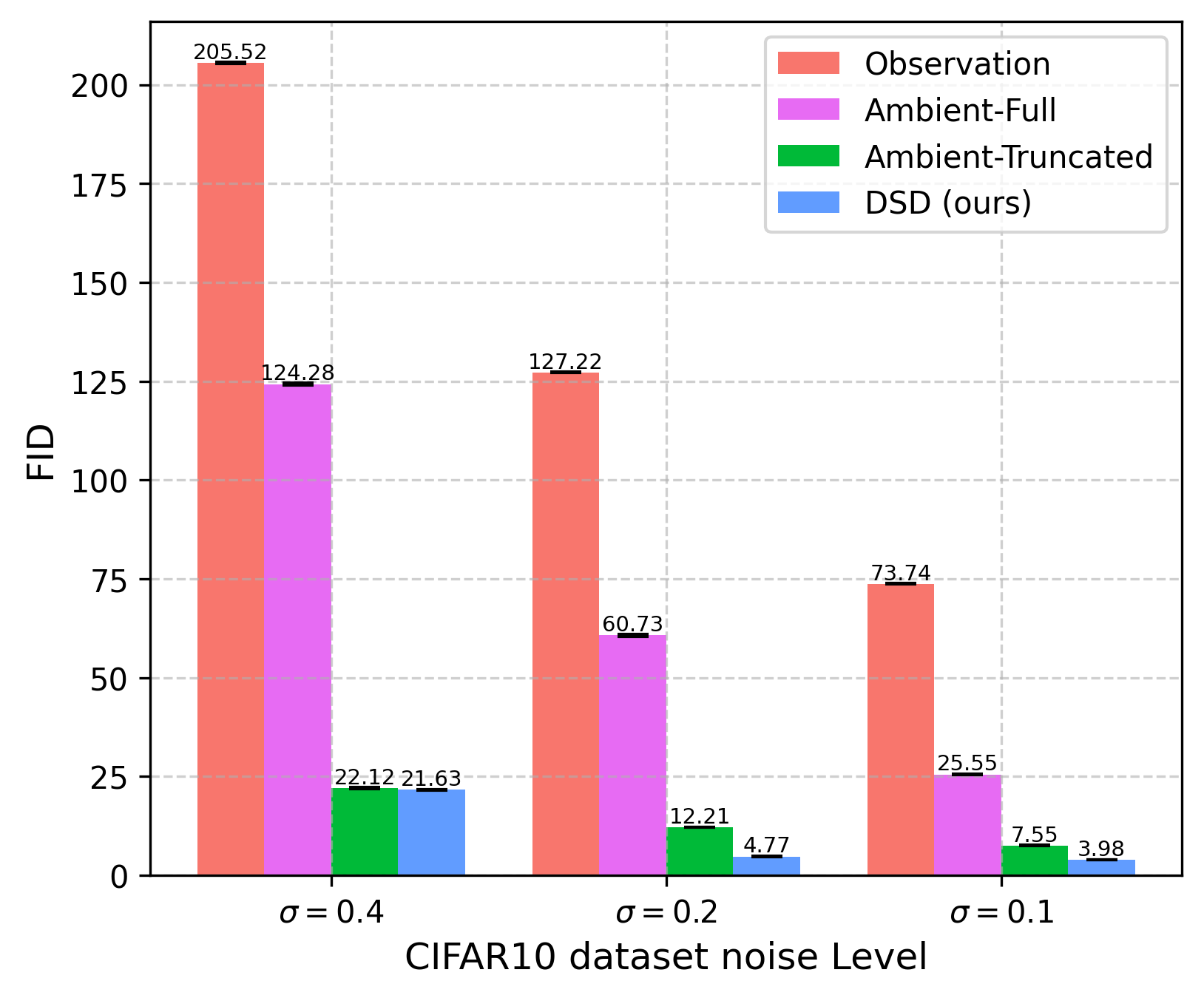}
        \caption{Evaluation of FID on CIFAR-10 across different noise levels.}
        \label{fig:fid_noise}
    \end{subfigure}
    \hfill
    \begin{subfigure}{0.32\textwidth}
        \centering
        \includegraphics[width=\linewidth]{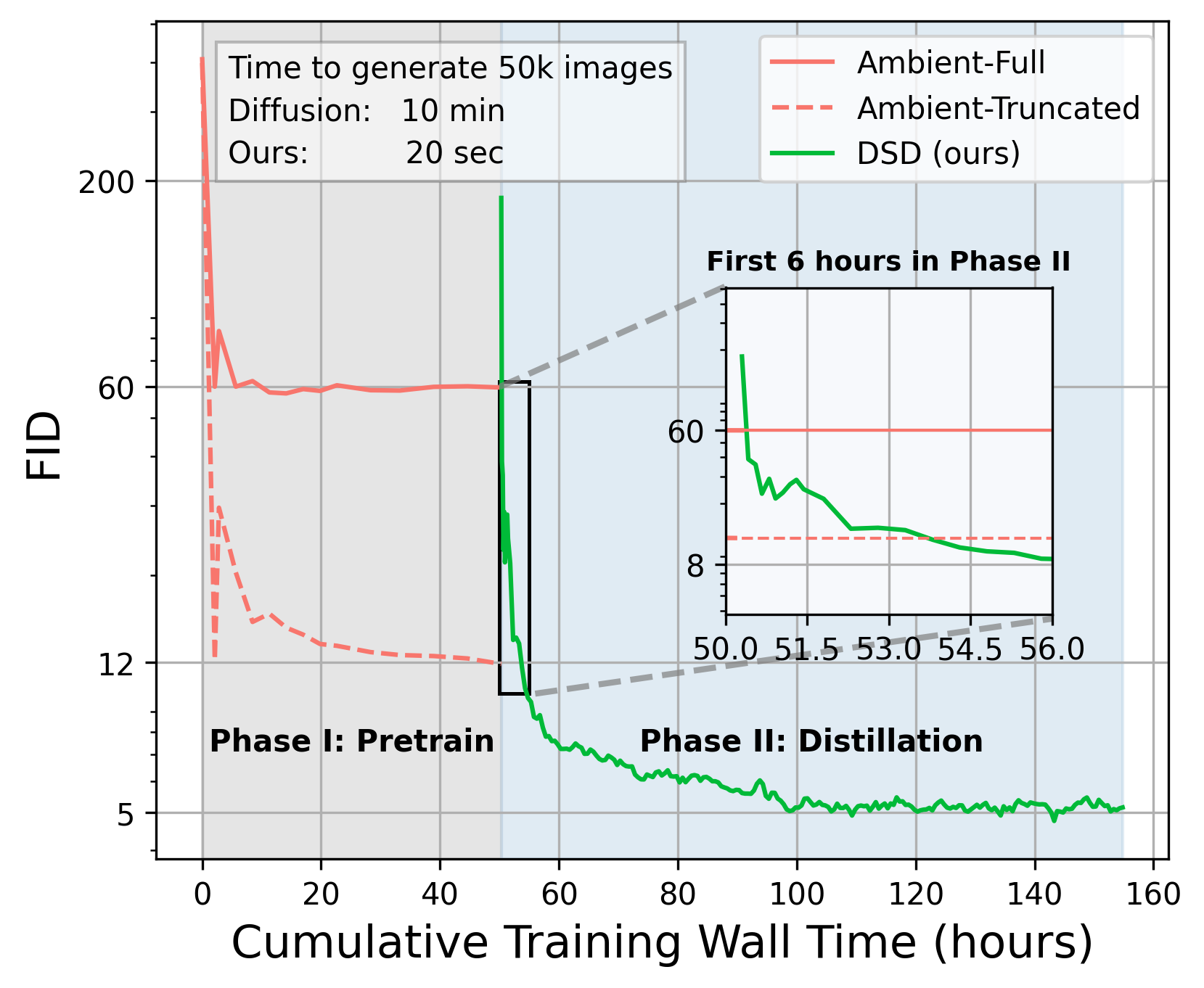}
        \caption{FID curve. Distillation within 4 hours surpasses teacher diffusion.}
        \label{fig:fid_pretrain_distill}
    \end{subfigure}
    \hfill
    \begin{subfigure}{0.32\textwidth}
        \centering
        \includegraphics[width=\linewidth]{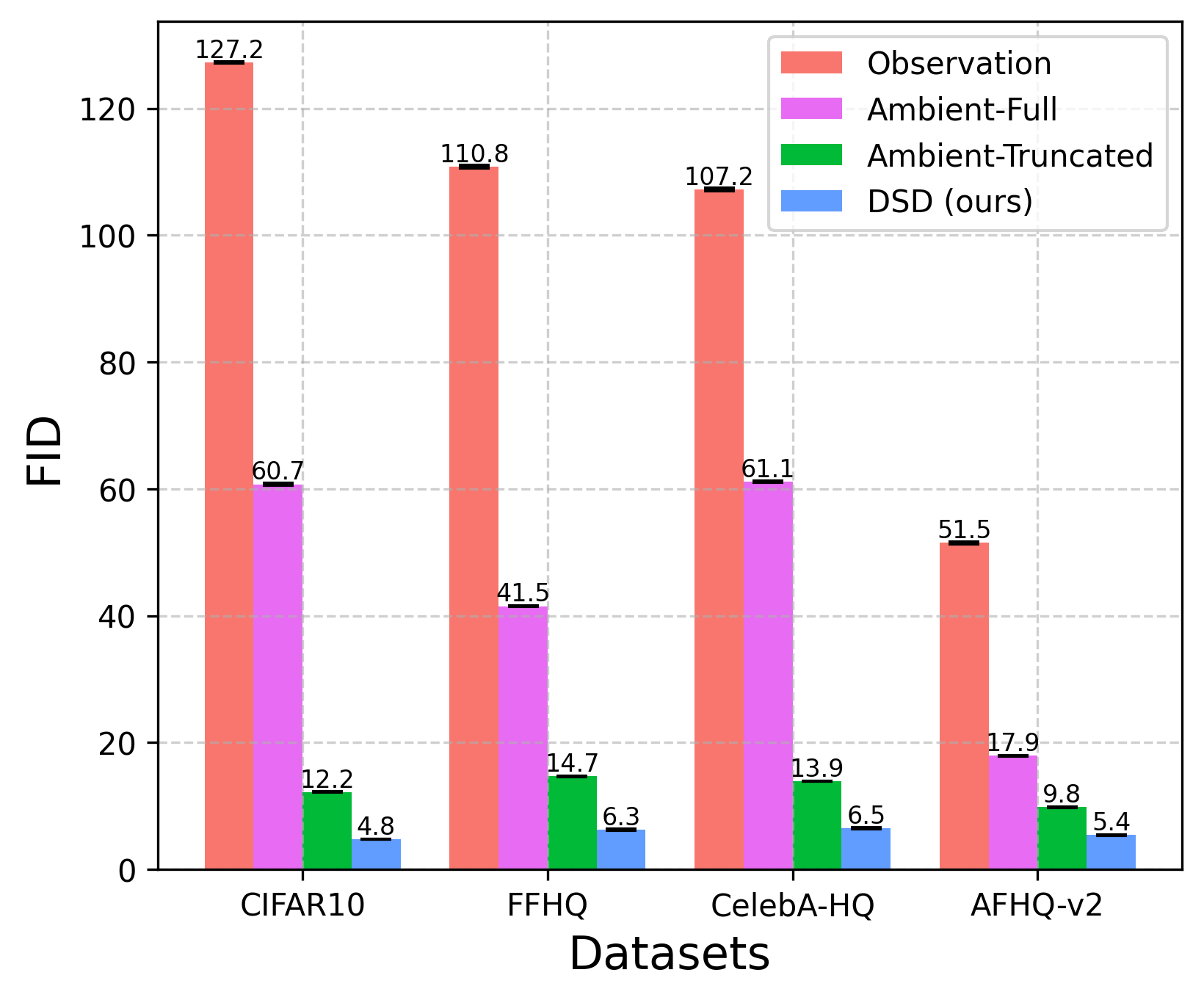}
        \caption{Evaluation of FID across different training datasets.}
        \label{fig:fid_datasets}
    \end{subfigure}
    \caption{ 
\textbf{Distilled student models (\textcolor{ours}{DSD, one-step}) surpass teacher diffusion models (\textcolor{full}{Ambient-Full} and \textcolor{trun}{Ambient-Truncated}) on FID in the following settings:}  
(a) Varying noise levels on CIFAR-10, and (c)  The same noise level across various datasets, including CIFAR-10, FFHQ, CelebA-HQ, and AFHQ-v2. For example, when $\sigma=0.2$, ours improves the FID from \textbf{12.21} to \textbf{4.77} on CIFAR-10. In addition,  
 (b) distillation within \textbf{4 hours} surpasses the teacher diffusion model.  
 Furthermore, distillation enjoys high inference efficiency and accelerates the generation of {50k images} from \textbf{10 minutes} to \textbf{20 seconds}, achieving a \textbf{30$\times$ speedup}.}

    \label{fig:first_plot}
\end{figure}

\end{abstract}    
\section{Introduction}
\label{sec:intro}

\begin{figure}[t]
    \centering
\includegraphics[width=1\linewidth]{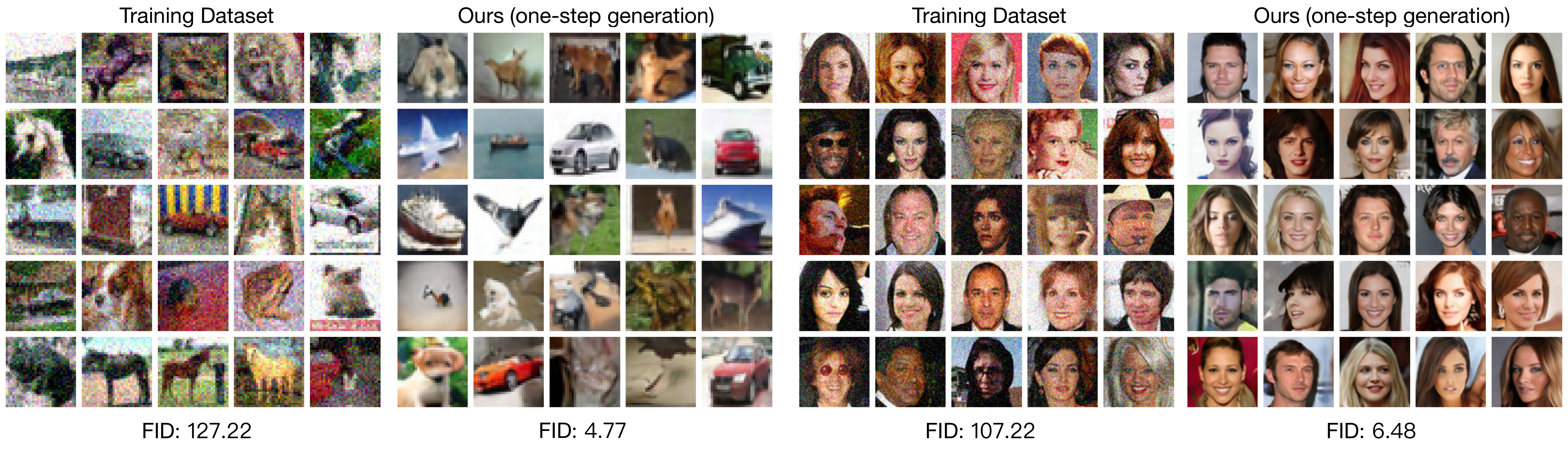}
    \caption{\textbf{Qualitative results of DSD (ours, one-step) at $\sigma=0.2$.} While only corrupted images are available during training, DSD is capable of producing refined, clean samples. The left two panels are from CIFAR-10, while the right two are from CelebA-HQ.  Zoom in for better viewing.}
    \label{fig:qual_results}
\end{figure}

Diffusion models \cite{sohl2015deep, ho2020denoising}, also known as score-based generative models \cite{song2019generative, song2021scorebased}, have emerged as the de facto approach for generating high-dimensional continuous data, particularly high-resolution images \cite{dhariwal2021diffusion, ho2022cascaded, ramesh2022hierarchical, rombach2022high, saharia2022photorealistic, peebles2023scalable, zheng2024learning, zhang2024object, chang2025skews}. These models iteratively refine random noise through diffusion processes, effectively capturing the complex distributions of their training data.

However, their performance is highly dependent on large-scale, high-quality datasets such as ImageNet \cite{deng2009imagenet}, LAION-5B \cite{schuhmann2022laion}, and DataComp \cite{gadre2023datacomp}. Constructing such datasets is an expensive and complex process \cite{gadre2023datacomp, baldridge2024imagen}. Moreover, this reliance on pristine data limits the applicability of diffusion models in scientific domains where clean data is scarce or costly to obtain, such as astronomy \cite{roddier1988interferometric, lin2024imaging}, medical imaging \cite{reed2021dynamic, jalal2021robust}, and seismology \cite{nolet2008breviary, rawlinson2014seismic}. For instance, in black-hole imaging, it is inherently impossible to obtain full measurements of the object of interest \cite{sun2021deep, leong2023discovering, lin2024imaging}. Additionally, training directly on original datasets containing private or copyrighted content, such as facial images may lead to ethical and legal issues \cite{carlini2023extracting, somepalli2023understanding, daras2023ambient}.  

To address these challenges, there has been growing interest in training generative models under corruption, where the available data is blurry, noisy, or incomplete \cite{bora2018ambientgan, kelkar2024ambientflow}. One classical approach leverages Stein’s Unbiased Risk Estimate (SURE) \cite{stein1981estimation-sure} to jointly learn an image denoiser and a diffusion model \cite{aali2023solving-sure, kawar2024gsurebased}. Another line of work explores {Ambient Diffusion} \cite{daras2023ambientdiffusion} and Ambient Tweedie~\cite{daras2025ambientscaling}, which train diffusion models from certain linear measurements. A different approach, EM-Diffusion \cite{weiminbai2024emdiffusion}, is based on the Expectation-Maximization (EM) algorithm \cite{dempster1977maximum},  alternating between reconstructing clean images from corrupted data using a known diffusion model and refining model weights based on these reconstructions. However, EM-Diffusion requires clean images initialization to effectively guide the learning process, which may not always be available in practical scenarios. Similarly, SFBD \cite{lu2025stochastic} frames the task of estimating the clean distribution as a density deconvolution problem \cite{meister2009deconvolution}, achieving decent performance with the help of clean data initialization.

In this work, we propose a surprisingly effective and novel approach for training high-quality generative models from low-quality data:  {denoising score distillation (DSD)}. Our method first pretrains diffusion models solely on noisy, corrupted data and then distills them into a one-step generator capable of producing refined, clean samples. While diffusion models typically suffer from the inefficiency of multi-step sampling, recent efforts have sought to accelerate them through advanced numerical solvers for stochastic and ordinary differential equations (SDE/ODE) \cite{song2020denoising, lu2022dpm, karras2022elucidating, liupseudo, zhao2024unipc, lu2022dpmplus} and distillation techniques \cite{song2023consistency, zhou2024score, yin2024one, salimans2022progressive, poole2022dreamfusion, xie2024em}. A prevailing view is that distillation primarily serves as a means to accelerate diffusion generation with minimal loss of output quality. However, as evidenced in Figs. \ref{fig:first_plot} and \ref{fig:qual_results}, we challenge this assumption by demonstrating that \textbf{score distillation \cite{poole2022dreamfusion, wang2024prolificdreamer, luo2023diff, yin2024one, zhou2024score, xie2024em} can, in fact,  {enhance} sample quality, particularly when the teacher model is trained on degraded data.} Our results suggest that  {{noisy data, when leveraged effectively through score distillation, can be more valuable than traditionally assumed}.}

To explain this phenomenon, we provide a theoretical analysis in Sec. \ref{sec:theory}, showing that, in a linear model setting, a distilled student model learns superior representations by aligning with the eigenspace of the underlying clean data distribution’s covariance matrix when given access to the noisy distribution's score. This insight reframes score distillation not only as an acceleration tool but also as a theoretically grounded mechanism for improving generative models trained on noisy data.
\textbf{We hope that our method, along with its theoretical framework, will inspire further research into leveraging distillation  for training generative models from corrupted data.}  

Our key contributions can be summarized as follows:

\begin{itemize}
    \item \textbf{Denoising Score Distillation for Learning from Noisy Data:} We introduce a novel training paradigm, DSD, which enables high-quality generative modeling from low-quality, noisy data by leveraging score distillation. Our approach highlights the potential of distillation in scenarios where clean data is scarce or unavailable.
   
\item \textbf{Empirical Evidence of Quality and Efficiency Enhancement:} We provide comprehensive experiments demonstrating that score distillation, contrary to conventional expectations, can significantly improve sample quality when applied to degraded teacher models. Quantitative results are shown in Fig. \ref{fig:first_plot} and Tab. \ref{tab:main_table}, while qualitative results are in Fig. \ref{fig:qual_results} and the Appendix. Besides, our method enjoys both training and inference efficiency as shown in Fig. \ref{fig:first_plot}(b) and Tab. \ref{tab:efficiency}. 

\item 
\textbf{Theoretical Justification for Student Model Superiority}: We provide a theoretical analysis in Sec. \ref{sec:theory} showing that in a linear model setting, a distilled student model can surpass a low-quality teacher by better capturing the eigenspace of the clean data distribution's covariance matrix, even when only given access to the noisy distribution's score. This insight offers a principled explanation for the observed improvements and establishes a new perspective on the role of distillation in generative learning.
\end{itemize}

\section{Background}

\subsection{Diffusion Models}

Diffusion models \cite{sohl2015deep, ho2020denoising}, also known as score-based generative models \cite{song2019generative, song2021scorebased}, consist of a forward
process that gradually injects noise to the data distribution and a reverse process that progressively
denoises the observations to recover the original data distribution $p_X(x)$. This results in a sequence
of noise levels $t \in (0,1]$ with conditional distributions $q_t(x_t | x) = \mathcal{N} (\alpha_t x, \sigma_t^2 I)$, whose marginals
are $q_t(x_t)$. We use a variance-exploding \cite{song2019generative} forward process such that $\alpha_t=1$ for simplicity, i.e., $x_t = x+\sigma_t \epsilon$ and $\epsilon \sim \cN(0,I_d)$. To learn the reverse diffusion process, extensive works \cite{karras2022elucidating} have considered training a time-dependent denoising autoencoder (DAE) 
$f_{\phi}(\cdot, t) : \mathbb{R}^d \times [0,1] \to \mathbb{R}^d$ \cite{vincent2011connection}
parameterized by a neural network with parameters $\phi$ to estimate the posterior mean 
$\mathbb{E} [x | x_t]$. To determine the parameters $\phi$, we can minimize the following empirical loss:
\begin{align}\label{eq:diffusion}
  \ell(\phi; \{x^{(i)}\}_{i=1}^N) := \frac{1}{N} \sum_{i=1}^{N} \int_{0}^{1} \mathbb{E}_{\epsilon \sim \mathcal{N}(0, I_d)}
\left[ \left\| f_{\phi}(x^{(i)} + \sigma_t \epsilon, t) - x^{(i)} \right\|^2_2 \right] dt,
\end{align}
where $\{x^{(i)}\}_{i=1}^N$ are $N$ observed data points, with $t$-dependent weighting functions omitted for~brevity.

\subsection{Score Distillation Methods}
Score distillation  initially emerged in 3D tasks \cite{poole2022dreamfusion, wang2024prolificdreamer} before being adapted for 2D image generation \cite{luo2023diff, yin2024one,  zhou2024score,  xie2024em}. This approach aims to compress a pretrained diffusion model into a one-step generator \( G_\theta: \mathbb{R}^d \to \mathbb{R}^d \). 
The generator is optimized to ensure that its induced distribution $(G_\theta)_{\sharp}(\cN(0,I_d))$\footnote{$G_{\sharp}(P)$ is the push-forward distribution of $P$ induced by a function $G$, i.e., $x \sim G_{\sharp}(P)$ if and only if $x = G(z),\ z\sim P.$} closely matches that of the pretrained teacher diffusion model, parametrized as $p_{\phi,t}(x_t)$, across all noise levels. Practically, \cite{luo2023diff, yin2024one,  zhou2024score,  xie2024em} include a fake diffusion $f_\psi(\cdot,t) : \mathbb{R}^d \times [0,1] \to \mathbb{R}^d$ to learn the distribution of the one-step generation at each noise level, denoted as $p_{\psi,t}(x_t)$. 
Given a fixed teacher  $\phi$, this alignment is achieved by minimizing the following objective:
\begin{align}
     \cJ (\theta;  \psi) = \mathbb{E}_{z\sim \cN(0,I_d), ~x = G_\theta(z)} \left [\int_{0}^1     \cD(p_{\psi,t} (x_t), p_{\phi, t} (x_t))   dt\right ], 
  \label{eq:distillation}
\end{align}
where $\mathcal{D}$ represents a divergence measure. We omit the weighting functions at different times $t$
 for brevity. Note that   the training objective of the fake diffusion $f_\psi$ is identical to training the teacher diffusion model as in Eq. \eqref{eq:diffusion}, except that the data comes from the generator $G_\theta$ rather than the dataset. In other words, it can be denoted as  
$\mathcal{J}(\psi;\theta) = \ell(\psi; \tilde{x})$ ,
  where $\tilde x \sim (G_\theta)_{\sharp}(\cN(0,I_d))$.

 Different score distillation approaches employ distinct choices of $\mathcal{D}$: for   Variational Score Distillation (VSD) \cite{wang2024prolificdreamer}, Diff-Instruct \cite{luo2023diff} and Distribution Matching Distillation (DMD) \cite{yin2024one}, $\mathcal{D}$ corresponds to the Kullback-Leibler (KL) divergence, whereas for Score identity Distillation (SiD) \cite{zhou2024score}, it is given by the Fisher divergence.  Note that the idea of distribution matching performed in the noisy space at multiple different noise levels aligns with Diffusion-GAN \cite{wang2022diffusion} where the Jensen-Shannon divergence is used.

 


\subsection{Ambient Tweedie}

Assume that we only have access to  noisy image samples $y = x + \sigma\epsilon 
\sim p_Y$ at a specific noise level $t_\sigma \in (0,1]$. Ambient Tweedie \cite{daras2024ambienttweedie} provides a method for learning an unbiased score for clean data from noisy data by utilizing Tweedie's formula \cite{efron2011tweedie}. Note that  learning a diffusion model $f_\phi$ is equivalent to learning the score function $\nabla_x \log p_X(x)$ at different time steps \cite{song2021scorebased}. As a result, instead of minimizing Eq. \eqref{eq:diffusion}, one can minimize
\begin{align}
   \ell_{\text{tweedie}}(\phi;\{y^{(i)}\}_{i=1}^N ) = \frac{1}{N} \sum_{i=1}^{N} \int_{t_\sigma}^{1}    \mathbb{E}_{\epsilon \sim \cN(0,I_d)}
\left\| \frac{\sigma_t^2 -\sigma^2}{\sigma_t^2 } f_\phi(x_t^{(i)}, t) + \frac{\sigma^2}{\sigma_t^2}  x_t^{(i)} - y^{(i)}\right\|^2dt,  
\label{eq:ambient_diffusion}
\end{align}
where $\{y^{(i)}\}_{i=1}^N$ are the observed $N$ noisy data points, and $x_t^{(i)} = y^{(i)} + \sqrt{\sigma_t^2 - \sigma^2} \epsilon$.  The loss can be seen as an adjusted diffusion objective with adaptation to noisy datasets at noise level $t_\sigma$.


\section{Denoising Score Distillation}
\label{sec:method}

\paragraph{Problem Statement.} Suppose we have a fixed corrupted, noisy dataset of size $N$, i.e. $\{y^{(i)}\}_{i=1}^N$. Assume that for each data point, $y^{(i)} = x^{(i)} + \sigma\epsilon^{(i)}$, where $\sigma$ is a known noise level and $\epsilon^{(i)} \overset{\text{i.i.d.}}{\sim} \mathcal{N}(0, I_d)$. Note that during training, we do not have access to the clean data $\{x^{(i)}\}_{i=1}^N$. 
The goal of our method is to recover the underlying clean distribution $p_X$  by learning a generator $G_\theta(\cdot)$.

\begin{algorithm}[t]
\caption{Denoising Score Distillation (DSD)}
\label{alg:ambient_distill}
\begin{algorithmic}[1]
\Procedure{Denoising-Score-Distillation}{$\{y^{(i)}\}_{i=1}^N, \sigma, K$}
\State \textbf{\textit{\# Phase I: Denoising Pretraining}}
    \State Pretrain $f_\phi$ using    \eqref{eq:diffusion} or \eqref{eq:ambient_diffusion} with the noisy training dataset $\{y^{(i)}\}_{i=1}^N$
    \State \textbf{\textit{\# Phase II: Denoising Distillation}}
    \State Initialize  fake diffusion model $f_\psi \leftarrow f_\phi$, one-step generator $G_\theta \leftarrow f_\phi$
    \For{$j = 1, \dots, K$}
        \State $x_g \leftarrow G_\theta(z), \quad z \sim \mathcal{N}(0, I_d)$ 
        \State $\tilde{y} \leftarrow x_g + \sigma \epsilon, \quad \epsilon \sim \mathcal{N}(0, I_d)$ \Comment{Corrupt fake image $x_g$.}
        \State Update  $f_\psi$ via a gradient step using     \eqref{eq:diffusion} or \eqref{eq:ambient_diffusion} with $\tilde y$
        \State $x_g \leftarrow G_\theta(z), \quad z \sim \mathcal{N}(0, I_d)$ 
        \State $\tilde{y} \leftarrow x_g + \sigma \epsilon, \quad \epsilon \sim \mathcal{N}(0, I_d)$ \Comment{Corrupt fake image $x_g$.}
        \State Update $G_\theta$ via a gradient step using the estimated generator loss \eqref{eq:distillation} with $\tilde y$ or $x_g$
    \EndFor
\EndProcedure
\end{algorithmic}
\begin{algorithmic}[1]
\Procedure{Denoising-Score-Distillation-Generation}{}
\State \textbf{\textit{\# One-Step Generation}}
\State $x_g=G_{\theta}(z),\quad z \sim \mathcal{N}(0, I_d)$
\EndProcedure
\end{algorithmic}
\end{algorithm}

Score distillation traditionally follows a two-phase approach: \textbf{Phase I}, where a diffusion model is pretrained on the training dataset, and \textbf{Phase II}, where the pretrained model is distilled into a one-step student generator. However, our experiments reveal that directly applying standard diffusion \eqref{eq:diffusion} and score distillation \eqref{eq:distillation} to noisy data leads to suboptimal performance,  as shown in Fig. \ref{fig:ablation}. To address this challenge, we introduce essential modifications to both phases to accommodate our corrupted-data-only setting. Specifically, in \textbf{Phase I}, we aim to learn either the score of the noisy dataset or an adjusted score of clean data inferred from noisy observations. In \textbf{Phase II}, we explore both the standard diffusion loss and an adapted diffusion loss tailored to the characteristics of noisy data for training the fake diffusion model.  We summarize our approach in Algorithm~\ref{alg:ambient_distill}, with further details discussed below.

\paragraph{Phase I: Denoising Pretraining}  
In this phase, we train a diffusion model \( f_\phi(\cdot, t) \) using the noisy training dataset \( \{y^{(i)}\}_{i=1}^N \), which serves as the teacher model for distillation. There are two possible approaches for training the diffusion model, depending on different perspectives:  

\begin{enumerate}  
    \item Pretrain \( f_\phi \) using the standard diffusion objective \eqref{eq:diffusion} across noise levels in $(0,1]$. In this case, the model learns the score of the noisy dataset, i.e., \( \nabla_y \log p_Y(y) \).  
    \item Pretrain \( f_\phi \) using the adjusted diffusion objective  \eqref{eq:ambient_diffusion} across noise levels in $(t_\sigma,1]$, allowing it to learn an unbiased score of the clean data from the noisy dataset, i.e., \( \nabla_x \log p_X(x) \).  
\end{enumerate}  
Intuitively, there is no distinct advantage to either approach, as even with the adjusted loss \eqref{eq:ambient_diffusion}, the model cannot learn the score at a noise level below \( t_\sigma \). More details on training with \eqref{eq:ambient_diffusion} are provided in Algorithm \ref{alg:ambient_training} in Appendix \ref{app:alg}. However, the choice of pretraining method influences the fake diffusion and one-step generator objectives in the distillation phase, as discussed below.

\paragraph{Phase II: Denoising Distillation}  
The objective of the second phase is to distill the pretrained teacher diffusion model from \textbf{Phase I} into a one-step generator. During distillation, the generator \( G_\theta \) is trained to produce clean images. Beyond standard score distillation, our method further corrupts the generated samples into \(\tilde{y}\) in the same manner as the corruption of the training dataset, as illustrated in Line 8/11 of Algorithm \ref{alg:ambient_distill}. These corrupted samples are then used to train a fake diffusion model \( f_{\psi}(x_t, t) \) and a one-step generator $G_\theta(z)$.  

The training of \( f_{\psi} \) and $G_\theta$ depends on the pretraining approach used in \textbf{Phase I}. Specifically, we consider two scenarios:  
\begin{enumerate}  
    \item If \( f_\phi \) is pretrained using \eqref{eq:diffusion}, then \( f_\psi \) is trained with the standard diffusion objective \eqref{eq:diffusion} on the dataset \(\tilde{y}\) across all noise levels \( t \in (0,1] \).  Furthermore, $\tilde{y}$ is regarded as the generated sample by $G_\theta$ and is used to estimate the generator loss, i.e., $x_t = \tilde y + \sigma_t \epsilon$ in  \eqref{eq:distillation}.
    
    \item If \( f_\phi \) is pretrained using \eqref{eq:ambient_diffusion}, then \( f_\psi \) is trained with the adjusted diffusion objective \eqref{eq:ambient_diffusion} on \(\tilde{y}\) over noise levels \( t \in (t_\sigma,1] \). Furthermore, $x_g$ is the generated sample by $G_\theta$ and used to estimate the generator loss, i.e., $x_t = x_g + \sigma_t \epsilon$ in  \eqref{eq:distillation}.
\end{enumerate}  
Note that maintaining consistency between the pretrained diffusion model, the fake diffusion model, and the one-step generator training objectives is essential. We call the first choice the standard diffusion way and the second choice the adjusted diffusion way.

  \begin{wrapfigure}{l}{0.48\textwidth} 
    \centering
    \includegraphics[width=0.45\textwidth]{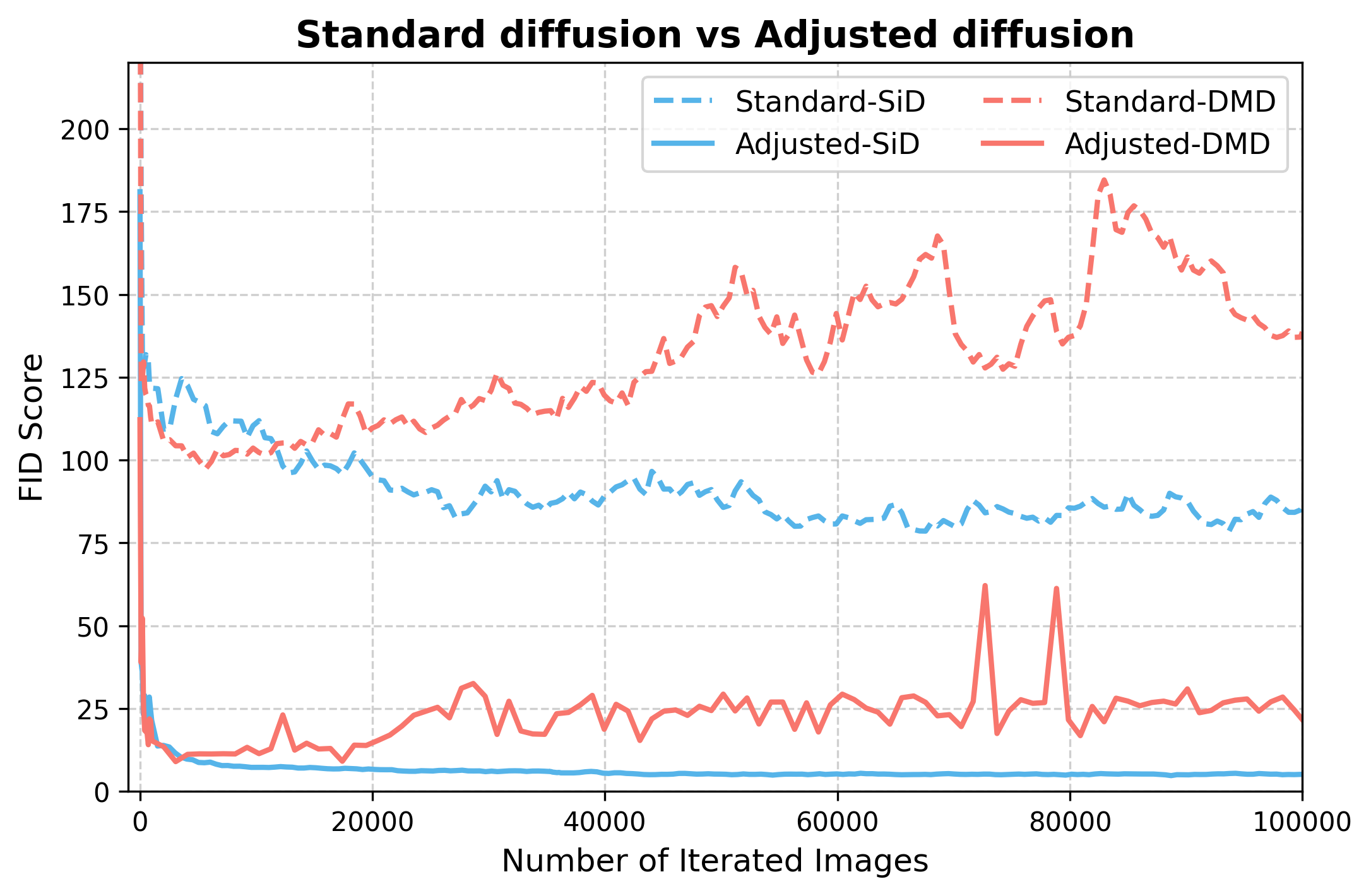} 
    \caption{\textbf{Ablation on diffusion objectives and generator losses.} The adjusted diffusion objective leads to excellent performance, while the Fisher divergence with SiD-based gradient estimation helps stabilize the distillation process. }
    \label{fig:ablation}
\end{wrapfigure}

In practice, for the optimization of the generator, the expectation in \eqref{eq:distillation} is estimated using sampled values of \( t \). Depending on the specific score distillation method employed, different divergence objectives  are selected accordingly, as formulated in  \eqref{eq:distillation}. A detailed loss formulation with specific gradient estimation methods is provided in Appendix \ref{app:distillation}.

In Sec. \ref{sec:experiment}, we present empirical distillation results evaluating different pretraining and fake score training objectives and generator losses in Tab. \ref{tab:cifar_results}. 
Our findings (Fig. \ref{fig:ablation}) consistently highlight that, across diverse datasets and noise levels, \textbf{integrating an adjusted diffusion learning objective with noise-aware adaptation, alongside selecting Fisher divergence as the generator loss and employing SiD-based gradient estimation, is crucial for stabilizing the distillation process and generating high-quality, clean outputs.}

\paragraph{{A Toy Example.}}
To intuitively understand why distillation works, we show a toy example in Fig. \ref{fig:toy}  to demonstrate the implicit regularization effects of distillation on the learned data distribution. 
The core issue with teacher diffusion models on noisy datasets, i.e., (b) Ambient-Full and (c) Ambient-Truncated, is that they force the approximating distribution to spread its probability mass across all regions where the target distribution has mass. 
In contrast, 
(d) DSD excels at denoising the original dataset, producing a narrow, concentrated, and sharp approximation. 

\section{Theory}\label{sec:theory}

For our theory, we aim to show that distilling a noisy teacher model can induce a distribution closer to the underlying clean data distribution. We will assume that our data follows a low-rank linear model for our analysis.

\begin{assumption}[Linear Low-Rank Data Distribution]
\label{assump:linear}
    Suppose our underlying data distribution is given by a low-rank linear model
$
x = Ez \sim p_X$ and $z \sim \mathcal{N}(0,I_r),
$
where \(E \in \mathbb{R}^{d \times r}\) with \(r < d\) and with orthonormal columns (i.e., \(E^TE=I_r\)). 
\end{assumption}

 Assumption \ref{assump:linear} is equivalent to $p_X := \mathcal{N}(0,EE^T)$.
For a fixed corruption noise level \(\sigma > 0\), consider the setting we only have access to the noisy distribution
$
y = x + \sigma \epsilon$, where $x \sim p_X$ and $\epsilon \sim \mathcal{N}(0,I_d)$. In other words, $p_{Y,\sigma} := \mathcal{N}(0,EE^T+\sigma^2 I_d).
$
In our setting we assume that we have perfectly learned the noisy score:

\begin{assumption}[Perfect Score Estimation]
\label{aasump:perfect_score}
    Suppose we can estimate the score function of corrupted data $y$ perfectly:
    $$
\nabla \log p_{Y,\sigma}(x) = -\bigl(EE^T+\sigma^2I_d\bigr)^{-1}x.
$$
\end{assumption}

Our goal is to distill this distribution into a distribution 
$
p_{G_\theta} := (G_\theta)_\sharp (\mathcal{N}(0,I_d))
$
given by the push-forward of \(\mathcal{N}(0,I_d)\) by a generative network \(G_\theta : \mathbb{R}^d \to \mathbb{R}^d\). To model a U-Net \cite{ronneberger2015u} style architecture with bottleneck structure, we assume $G_\theta$ satisfies the following low-rank linear structure detailed in Assumption \ref{assump:low_rank_generator}.

\begin{assumption}[Low-Rank Linear Generator]
    \label{assump:low_rank_generator}
    Assume the generator is a low-rank linear mapping, where \( G_\theta \) is parameterized by \( \theta = (U, V) \) where $U,V \in \mathbb{R}^{d \times r}$ with $r < d$ and has the form:
    $$
    G_\theta(z) := U V^T z.
    $$
\end{assumption}

 Note that $G_{\theta}$ induces a degenerate low-rank Gaussian distribution $p_{G_{\theta}} := \mathcal{N}(0,UV^TVU^T)$. Consider a bounded noise schedule $(\sigma_t) \subseteq [\sigma_{\min},\sigma_{\max}]$ for some $0 < \sigma_{\min} < \sigma_{\max} < \infty$ and perturbed data points $x_t = x + \sigma_t \epsilon$ where $\epsilon\sim\mathcal{N}(0,I_d)$ and $x\sim p_{G_{\theta}}$. Then $x_t \sim p^{\sigma_t}_{G_{\theta}} := \mathcal{N}(0,UV^TVU^T+\sigma_t^2I_d)$. To distill the noisy distribution, we minimize the score-based loss (or Fisher divergence) as in \cite{zhou2024score}:
\begin{equation}\label{eq:ideal-score-loss-r}
\mathcal{L}(\theta) := \mathbb{E}_{t \sim \mathrm{Unif}(0,1)}\mathbb{E}_{x_t \sim p^{\sigma_t}_{G_{\theta}}}\left[\left\|s_{\sigma,\sigma_t}(x_t) - \nabla \log p^{\sigma_t}_{G_{\theta}}(x_t) \right\|_2^2\right].
\end{equation} 

Here, $s_{\sigma,\sigma_t}(x):= -(EE^T + (\sigma^2 + \sigma_t^2)I_d)^{-1}x$. Note this objective is similar to Eq. \eqref{eq:distillation}, but with the real score in place of the fake score. This is also considered the idealized distillation loss (see Eq. (8) in \cite{zhou2024score}). In Theorem \ref{thm:time-dependent-theorem}, we show that minimizing Eq. \eqref{eq:ideal-score-loss-r} over a certain family of non-degenerate parameters finds a distilled distribution with \textbf{smaller} Wasserstein-2 distance to the underlying clean distribution. The formal proof is deferred to Appendix \ref{sec:appx-theorem-proof}.
\begin{theorem} \label{thm:time-dependent-theorem}
    Fix $\sigma > 0$. Under Assumptions \ref{assump:linear}, \ref{aasump:perfect_score}, and \ref{assump:low_rank_generator}, consider the family of parameters $\theta = (U, V)$ such that $$\theta \in \Theta := \{(U,V) : U^TU = I_r, V^TV \succ 0\}.$$ For any bounded noise schedule $(\sigma_t) \subseteq[\sigma_{\min},\sigma_{\max}]$, the global minimizers of $\mathcal{L}$ (Eq. \eqref{eq:ideal-score-loss-r}) over $\Theta$, denoted by $\theta^*_{\sigma} := (U^*,V^*_{\sigma})$, satisfy the following: \begin{align} 
    U^* = EQ\ \text{for some orthogonal matrix}\ Q\ \text{and}\ (V^*_{\sigma})^TV^*_{\sigma} = (1+\sigma^2)I_r. \label{eq:opt-representation}
    \end{align} For any such $\theta^*_{\sigma}$, the induced generator distribution $p_{G_{\theta^*_{\sigma}}} = \mathcal{N}(0,(1+\sigma^2)EE^T)$ satisfies $$W_2^2(p_{G_{\theta^*_{\sigma}}},p_X)
         = W_2^2(p_{Y,\sigma},p_X) - (d-r)\sigma^2 <  W_2^2(p_{Y,\sigma},p_X).$$
\end{theorem}

\begin{proof}[Proof sketch of Theorem \ref{thm:time-dependent-theorem}]
By directly computing the above expectation and using properties of the trace, one can show that there exists a constant $C_{\sigma,\sigma_t}$ independent of $\theta = (U,V) \in \Theta$ such that $$\mathcal{L}(\theta) = C_{\sigma,\sigma_t}+ B(U,V) + R(V),$$ where $$B(U,V) \propto - \mathrm{tr}(E^TUV^TVUE)$$ and $R(V)$ depends solely on the eigenvalues of $V^TV$. Here, $\propto$ refers to proportionality up to multiplicative constants that only depend on $\sigma,\sigma_t$. For any feasible $V$, minimizing $B(U,V)$ with respect to $U$ corresponds to maximizing the following quantity, which is akin to PCA: $$U \mapsto \mathrm{tr}(E^TUV^TVU^TE).$$ By exploiting the von Neumann trace inequality \cite{mirsky1975trace}, one can show that any maximizer of this quantity is of the form $U^* := EQ$ for some orthogonal matrix $Q$. Note that this maximizer does not depend on $V$. Plugging this back into $B(U^*,V)$ gives a quantity that only depends on the eigenvalues of $V^TV$. One can then show in order for $V^*_{\sigma}$ to minimize $B(U^*,V) + R(V)$, all eigenvalues of $(V^*_{\sigma})^TV^*_{\sigma}$ to be equal to $1+\sigma^2$. The Spectral Theorem guarantees that this implies $(V^*_{\sigma})^TV^*_{\sigma} = (1+\sigma^2)I_r.$ This completes the argument for $\theta^*_{\sigma}$ in Eq. \eqref{eq:opt-representation}. The final Wasserstein error bound is a direct computation.
\end{proof}


\section{Experiments} 
\label{sec:experiment}

In this section,  
we first present a toy example to display the implicit regularization effects brought by DSD, which will further be explained in Sec. \ref{sec:theory}. Then, extensive empirical evaluations on natural images are provided to validate the effectiveness of our proposed method. We conduct experiments on multiple datasets, including CIFAR-10 \cite{krizhevsky2009cifar}, FFHQ \cite{karras2019ffhq}, CelebA-HQ \cite{liu2015deepceleba, karras2017progressiveceleba}, and AFHQ-v2 \cite{choi2020afhq}, comparing our approach to various baselines in terms of sample quality, inference efficiency, and robustness to corruption levels.  Our results show that the distilled student models not only achieve faster inference but also consistently outperform their teacher models in terms of Fréchet Inception Distance (FID) \cite{heusel2017gans}, which quantifies the distance between the  distributions of generated samples  and clean data, as shown in Tabs. \ref{tab:cifar_results} and \ref{tab:3dataset}. Detailed implementation details are provided in Appendix \ref{app:implementation}.

\begin{figure}[t]
    \centering
    \begin{subfigure}{0.24\textwidth}
        \centering
        \includegraphics[width=\linewidth]{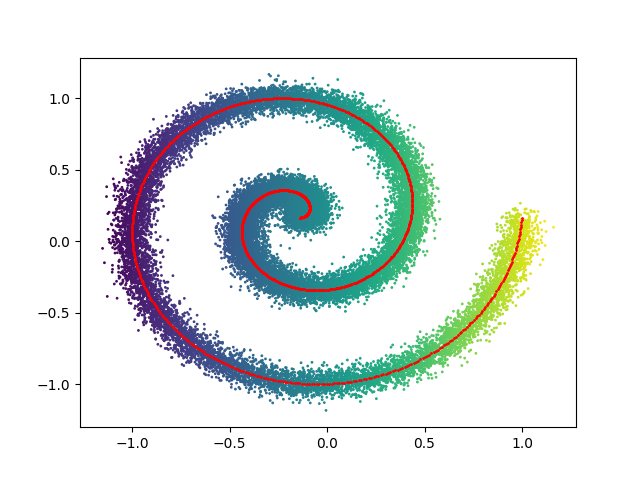}
        \caption{Noisy dataset $\sigma=0.05$}
        \label{fig:sub1}
    \end{subfigure}
    \begin{subfigure}{0.24\textwidth}
        \centering
        \includegraphics[width=\linewidth]{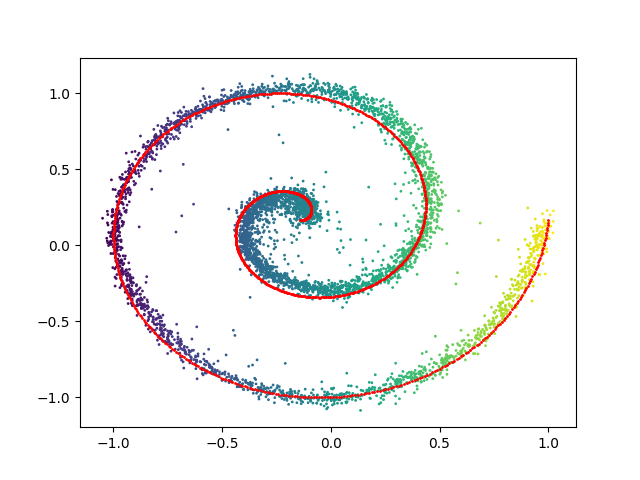}
        \caption{Ambient-Full}
        \label{fig:sub2}
    \end{subfigure}
    \begin{subfigure}{0.24\textwidth}
        \centering
        \includegraphics[width=\linewidth]{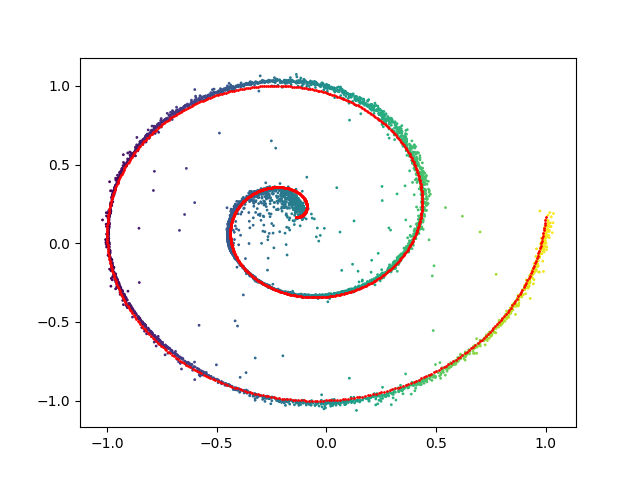}
        \caption{Ambient-Truncated}
        \label{fig:sub3}
    \end{subfigure}
    \begin{subfigure}{0.24\textwidth}
        \centering
        \includegraphics[width=\linewidth]{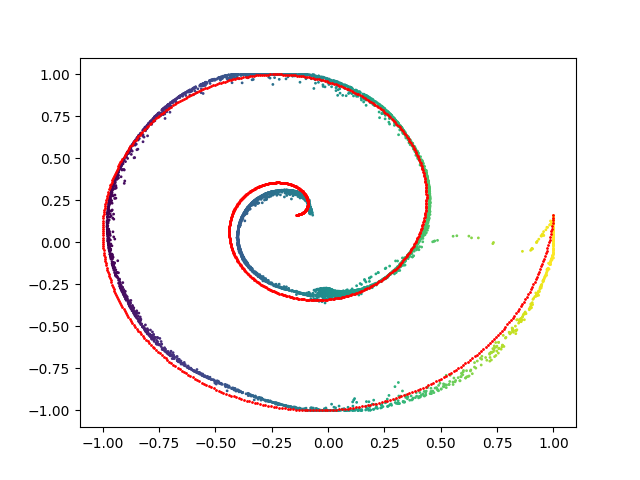}
        \caption{DSD (one-step)}
        \label{fig:sub4}
    \end{subfigure}
    \caption{\textbf{A toy example of learning from a noisy dataset with $\sigma=0.05$.}   Teacher diffusion models such as Ambient-Full and  Ambient Truncated tend to force the approximating distribution to spread out its probability mass to cover all regions. DSD excels at denoising the original dataset, demonstrating the implicit regularization effects brought by distillation.  }
    \label{fig:toy}
\end{figure}

\begin{table}[!t]
\centering
\caption{ 
\textbf{Results of our method DSD on CIFAR-10 and CelebA-HQ at noise level $\sigma=0.2$.} $\sigma=0.0$ refers to diffusion models trained on clean data, representing the upper-bound performance achievable by our method when trained exclusively on corrupted data. In contrast, $\sigma=0.2$ indicates models trained directly on noisy data. For methods employing few-shot initialization, 50 clean images are used. All baseline results are sourced from original papers or related works such as \cite{weiminbai2024emdiffusion} and \cite{lu2025stochastic}. Notably, the distilled student model DSD consistently outperforms its teacher diffusion models Ambient-Full and Ambient-Truncated, and even surpasses all baselines that rely on few-shot clean image initialization.
}
\label{tab:main_table}
\begin{tabular}{l|ccc|ccc}
\toprule
\multirow{2}{*}{\textbf{Methods}} & \multicolumn{3}{c|}{\textbf{CIFAR-10 (32$\times$32)}} & \multicolumn{3}{c}{\textbf{CelebA-HQ (64$\times$64)}} \\
\cmidrule(lr){2-4} \cmidrule(lr){5-7}
& $\boldsymbol{\sigma}$ & \textbf{Type} & \textbf{FID} & $\boldsymbol{\sigma}$ & \textbf{Type} & \textbf{FID} \\
\midrule
DDPM \cite{ho2020denoising}           & 0.0 & Full-Shot  & 4.04  & 0.0 & Full-Shot  & 3.26 \\
DDIM \cite{song2021ddim}         & 0.0 & Full-Shot  & 4.16  & 0.0 & Full-Shot  & 6.53 \\
EDM \cite{karras2022elucidating}         & 0.0 & Full-Shot  & {1.97} & -   & -   & -    \\ \midrule
SURE-Score \cite{aali2023solving-sure}   & 0.2 & Few-Shot & 132.61 & -   & -   & -    \\
Ambient Diffusion \cite{daras2023ambient}     & 0.2 & Few-Shot  & 114.13 & -   & -   & -    \\
EM-Diffusion \cite{weiminbai2024emdiffusion}        & 0.2 & Few-Shot & 86.47  & -   & -   & -    \\
TweedieDiff \cite{daras2024ambienttweedie} & 0.2 & Few-Shot & 65.21  & 0.2 & Few-Shot & 58.52 \\
{SFBD} \cite{lu2025stochastic}           & 0.2 & Few-Shot & {13.53} & 0.2 & Few-Shot & {6.49} \\ \midrule
TweedieDiff \cite{daras2024ambienttweedie} & 0.2 & Zero-Shot  & 167.23 & 0.2 & Zero-Shot  & 246.95 \\
Ambient-Full \cite{daras2025ambientscaling}& 0.2 & Zero-Shot & 60.73 & 0.2 & Zero-Shot & 61.14  \\
Ambient-Truncated \cite{daras2025ambientscaling}& 0.2 & Zero-Shot & 12.21 & 0.2 & Zero-Shot & 13.90 \\ 
Ambient-Consistency \cite{daras2025ambientscaling} & 0.2 & Zero-Shot & 11.93 & 0.2 & Zero-Shot & 12.97 \\     
DSD (Ours, One-Step) & 0.2 & Zero-Shot & \textbf{4.77} & 0.2 & Zero-Shot & \textbf{6.48} \\
\bottomrule
\end{tabular} 
\end{table}

\begin{table*}[t]
\centering
\begin{minipage}[t]{0.56\textwidth}
\centering
\caption{\textbf{Results of our methods (D-SDS, D-DMD, D-SiD) on CIFAR-10 at various noise levels.}    Note that the distilled student model D-SiD consistently surpasses the teacher diffusion models  {Ambient Full} and  {Ambient Truncated}.}
\label{tab:cifar_results}
\resizebox{\textwidth}{!}{
\begin{tabular}{l|c|ccc} 
\toprule
\textbf{Methods} & \textbf{Type} & \textbf{\boldmath$\sigma=0.1$} & \textbf{\boldmath$\sigma=0.2$} & \textbf{\boldmath$\sigma=0.4$} \\ 
\midrule 
Ambient-Full & Teacher & 25.55 & 60.73 & 124.28 \\
Ambient-Truncated & Teacher & 7.55 & 12.21 & 22.12 \\
\midrule
D-SDS & Distilled & > 200 & > 200 & > 200 \\
D-DMD & Distilled & 12.52 & 7.48 & 30.09 \\
D-SiD & Distilled & \textbf{3.98} & \textbf{4.77} & \textbf{21.63} \\
\bottomrule
\end{tabular}
}
\end{minipage}%
\hfill
\begin{minipage}[t]{0.42\textwidth}
\centering
\caption{\textbf{Results of our method DSD on FFHQ and AFHQ-v2 at $\sigma=0.2$.}   Our distilled model with only one-step generation surpasses the teacher diffusion models by a large margin across various datasets.}
\label{tab:3dataset}
\resizebox{ \textwidth}{!}{    \begin{tabular}{l|cc}
    \toprule
        \textbf{Methods} & \textbf{FFHQ} & \textbf{AFHQ-v2} \\ 
    \midrule
        Observation & 110.83 {\small$\pm$0.22} & 51.51 {\small$\pm$0.15} \\  
        Ambient-Full & 41.52 {\small$\pm$0.10} & 17.93 {\small$\pm$0.03} \\   
        Ambient-Truncated & 14.67 {\small$\pm$0.02} & 9.82 {\small$\pm$0.02} \\  
    \midrule
        DSD (one-step) & \textbf{6.29} {\small$\pm$0.15} & \textbf{5.42} {\small$\pm$0.08} \\ 
    \bottomrule
    \end{tabular}}
\end{minipage}
\end{table*}

\paragraph{Our Method.} As introduced in Sec. \ref{sec:method}, the framework of DSD involves 1) selecting the standard diffusion way or the adjusted diffusion way of training, and 2) selecting different objectives for the generator. In our initial experiments, we observed that pretraining and distilling with standard diffusion objectives led to suboptimal performance. As a result, we further experimented with the noise-aware adjusted diffusion loss \eqref{eq:ambient_diffusion} and found it crucial for excellent distillation performance. A sharp contrast in performance with different choices of diffusion objectives is demonstrated in Fig. \ref{fig:ablation}. Hereafter, we use the adjusted diffusion way by default.

For the generator loss, we explored three representative score distillation objectives: SDS, DMD, and SiD, denoted as {D-SDS}, {D-DMD}, and {D-SiD}, respectively.  D-SDS diverges at the initial stage of distillation, and its FID explodes quickly. The distillation process shown in Fig. \ref{fig:ablation} and the qualitative results in Tab. \ref{tab:cifar_results} indicate that D-SiD outperforms the other objectives in terms of both distillation stability and final performance metrics. Given its superior performance, we  refer to D-SiD as DSD for the remainder of our paper unless explicitly stated otherwise.

\paragraph{Performance Comparison with Baseline Methods.} We compare our method with various baselines as shown in Tab. \ref{tab:main_table}.   For brevity,   we categorize the  main baseline methods into three  groups: 1) Teacher Diffusion Models  \textbf{Ambient-Full} and \textbf{Ambient-Truncated } 
\cite{daras2025ambientscaling}: The teacher diffusion model trained with Ambient Tweedie loss \eqref{eq:ambient_diffusion} serves as a strong generative baseline, capable of producing clean images through reverse sampling.  We compare two sampling schemes (details are provided in Algorithm \ref{alg:ambient_sampling}):  
  a) \textbf{Ambient-Full}: This sampling scheme continues sampling until $t = 0$, adhering to the standard diffusion sampling  with a trained score function. 
b) \textbf{Ambient-Truncated}: This sampling scheme follows an early-stopping approach, where sampling terminates at $\sigma_t=\sigma$, where $\sigma$ is the predefined corruption level.  
\cite{daras2025ambientscaling} further introduces an additional consistency objective, termed \textbf{Ambient-Consistency}, to enhance the performance of trained models. 2)  Few-Shot Methods 
\textbf{EM-Diffusion} \cite{weiminbai2024emdiffusion} and \textbf{SFBD} \cite{meister2009deconvolution}:  EM-Diffusion   based on the Expectation-Maximization (EM) algorithm   alternates between 
reconstructing clean images from corrupted data using a known diffusion model via DPS \cite{chung2022diffusion} (E-step) and refining model 
weights based on these reconstruction (M-step). SFBD  frames the task of estimating the clean distribution as a density deconvolution 
problem.
Both methods require a small number of clean images as initialization to enable the training process. 
3) Standard Diffusion Models trained purely on clean data  \textbf{DDPM} \cite{ho2020denoising}, \textbf{DDIM} \cite{song2020denoising}, and \textbf{EDM} \cite{karras2022elucidating}: these three methods represent the upper-bound performance achievable by our method when trained exclusively on corrupted data. Our method, DSD, achieves the best performance among all baselines, including both zero-shot and few-shot methods, as shown in Tab. \ref{tab:main_table}.


\paragraph{Ablation on Datasets and Noise Levels.} Tab. \ref{tab:cifar_results} indicates that our method is effective and robust at all levels of noise, consistently surpassing teacher diffusion models. Tab.~\ref{tab:3dataset} presents the outstanding distillation results for more datasets on  FFHQ  and AFHQ-v2, while keeping the corruption level fixed at $\sigma=0.2$.

\begin{wraptable}{l}{0.5\linewidth}
 
\centering
\caption{Conditional inverse problem results of denoising on CIFAR10 at $\sigma=0.2$. Results for the baselines are taken from \cite{weiminbai2024emdiffusion}. We follow \cite{weiminbai2024emdiffusion} and sample 250 test images and compute the average PSNR and LPIPS. }
\resizebox{ 0.5\textwidth}{!}{\begin{tabular}{lccc }
\toprule
\textbf{Method} & \textbf{Type} & \textbf{PSNR$\uparrow$} & \textbf{LPIPS$\downarrow$} \\
\midrule
Observations         &      & 18.05 & 0.047  \\
DPS w/ clean prior \cite{chung2022diffusion}   & Full-Shot    & 25.91 & 0.010  \\ \midrule
SURE-Score \cite{aali2023solving-sure} &  Few-Shot   & 22.42 & 0.138  \\
AmbientDiffusion \cite{daras2023ambient}&  Few-Shot  & 21.37 & 0.033   \\
EM-Diffusion \cite{weiminbai2024emdiffusion}  &   Few-Shot        &  {23.16} &  {0.022}  \\ \midrule
Noise2Self \cite{batson2019noise2self} & Zero-Shot   & 21.32 & 0.227   \\ 
DSD (ours) & Zero-Shot   &   \textbf{24.11} &       \textbf{0.025}        \\ 
\bottomrule
\end{tabular}}
\label{tab:cifar10_denoising}
 \end{wraptable}

\paragraph{Solving Conditional Inverse Problems.}  
A promising future direction is to extend our framework into a conditional solver for inverse problems \cite{chung2022diffusion, zhang2024flow, zhu2024think, zhang2025learning},
 enabling applications in scientific and engineering domains. 
We provide a preliminary result of solving inverse problems with DSD in Tab. \ref{tab:cifar10_denoising}.  Our method surpasses previous zero-shot methods by a large margin and even achieves comparable performance with few-shot methods such as EM-Diffusion. In our implementation, we optimize $\min_z ||\cA (G_\theta(z)) - y||^2$ for 1000 steps using a learning rate of 0.05 with the Adam optimizer~\cite{kingma2014adam}.
We consider other methods to solve the inverse problems with one-step generators as future directions.

\paragraph{New Metric for Model Selection Criterion: Proximal FID.} In settings where only corrupted data is available, traditional FID metrics are unsuitable for model selection, as they measure the distance between clean and generated images. To address this, we propose Proximal FID, a new metric specifically designed for such scenarios.  
 The metric is computed as follows: we generate 50k images using the trained generator, corrupt them to match the noise level of the training dataset—yielding a batch of corrupted images, i.e.,  $
\{x_g^{(i)} + \sigma \epsilon^{(i)}\}_{i=1}^{50k},
$
and then calculate FID against the noisy training dataset. Our findings show that Proximal FID reliably selects models with ground-truth FID values near the optimal, as shown in Tab. \ref{tab:proximal}.  
We present the evolution of FID and Proximal FID results on CIFAR-10 in Fig. \ref{fig:proximal_fid} and other datasets in  Fig. \ref{fig:3dataset} in Appendix \ref{app:proximal}. 
In Tabs. \ref{tab:cifar_results} and Tabs. \ref{tab:3dataset}, we focus on providing a comparison on FID as this was the main metric used in previous baselines.
  We propose this metric for more realistic settings and hope to encourage the adoption of Proximal FID as a standard evaluation metric in the field of learning from corrupted data.

\begin{minipage}{0.48\textwidth} 
    \centering
\includegraphics[width=0.9\textwidth]{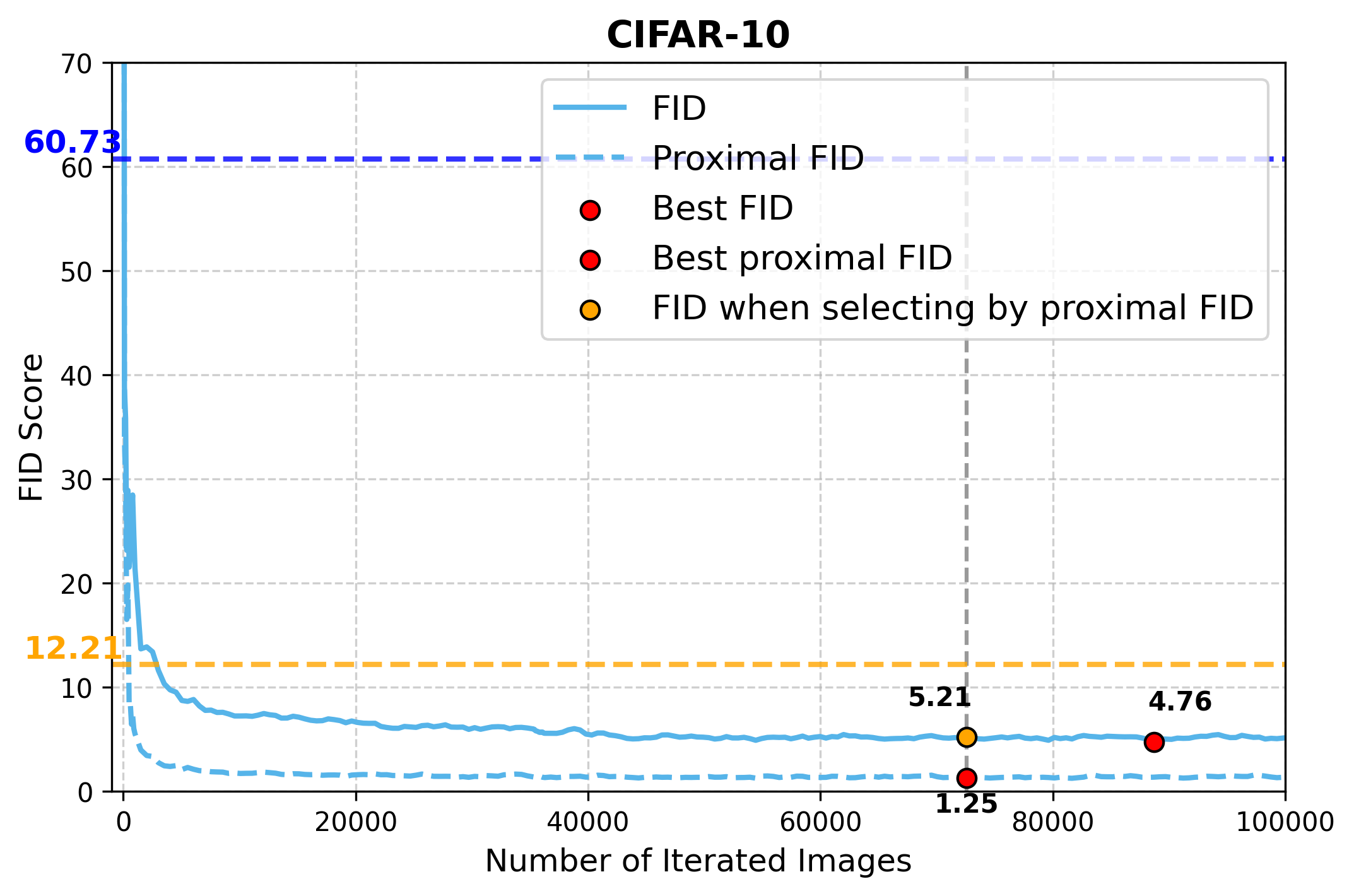}
    \captionof{figure}{\textbf{Evolution of FIDs and Proximal FIDs on D-SiD.} Proximal FID aligns well with FID throughout the distillation process.}
    \label{fig:proximal_fid}
\end{minipage}%
\hfill 
\begin{minipage}{0.48\textwidth} 
    \centering
    \captionof{table}{\textbf{The best FIDs selected by Proximal FID or FID of DSD.} Proximal FID serves as a reliable alternative to the true FID, consistently selecting models whose ground-truth FID is close to the best achievable FID.}
    \label{tab:proximal}
   \resizebox{ 0.8\textwidth}{!}{  \begin{tabular}{l|cc}
        \toprule
        \textbf{Datasets} & \textbf{Proximal FID} & \textbf{FID} \\ 
        \midrule
        CIFAR-10          & 5.21 \textcolor{red}{(+0.45)} & 4.76 \\
        FFHQ              & 6.12 \textcolor{red}{(+0.04)} & 6.08 \\
        CelebA-HQ         & 6.90 \textcolor{red}{(+0.54)} & 6.36 \\
        AFHQ-v2           & 5.45 \textcolor{red}{(+0.06)} & 5.39 \\   
        \bottomrule 
    \end{tabular}}
\end{minipage}

\section{Conclusion}
In this work, we introduced Denoising Score Distillation (DSD), a novel approach for training high-quality generative models from noisy, corrupted data using score distillation. 
Our empirical results demonstrate that DSD not only enhances sample fidelity across diverse datasets and noise levels, even under poor initial training conditions, but also achieves high training and inference efficiency. 
Furthermore, our theoretical analysis reveals that DSD implicitly regularizes the generator by identifying the eigenspace of the clean data distribution’s covariance matrix.
A detailed discussion of limitations is provided in Appendix \ref{app:limitation}.


{
    \small
    \bibliographystyle{abbrvnat}
    \bibliography{main}
}

\appendix
\clearpage
\setcounter{page}{1}
{\Large \textbf{Appendix}}
\section{Discussions and Limitations}\label{app:limitation}

\paragraph{Unknown Variance Size}  
In practical scenarios where only corrupted datasets are available, the true noise variance \(\sigma\) is often unknown. To address this challenge, we propose two potential solutions: (1) \textbf{Variance Estimation}, where \(\sigma\)  can be estimated  from a single image or an entire dataset \cite{donoho2002noising,zhou2009non,liu2013single,lebrun2013nonlocal,chen2015efficient},  and (2) \textbf{Hyperparameter Tuning}, where \(\sigma\) is treated as a tunable parameter optimized for the generative performance. In Fig. \ref{fig:sigma_estimate}, we present a toy example illustrating the impact of tuning 
$\sigma$. We use a noisy training dataset with $\sigma=0.05$. During pretraining and distillation, we experiment with different values of $\hat \sigma$, representing underestimation, accurate estimation, and overestimation. A slight overestimation of the noise level tends to increase regularization strength, helping the generated data better adhere to the data manifold. 


\begin{figure}[htbp]
    \centering
    \begin{subfigure}{0.32\textwidth}
        \centering
        \includegraphics[width=\linewidth]{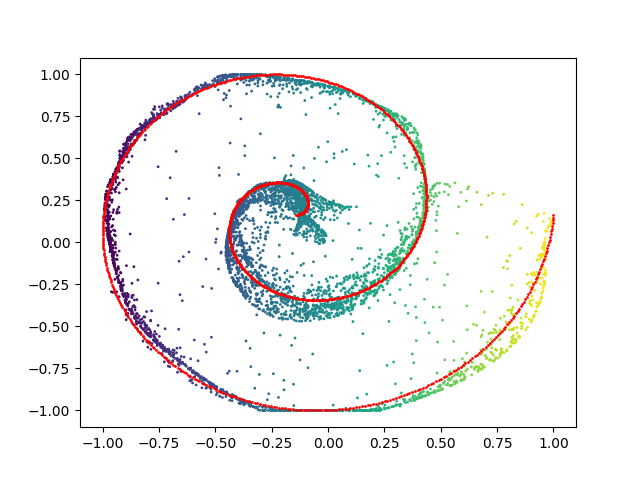}
        \caption{$\hat\sigma=0$}
    \end{subfigure}
    \begin{subfigure}{0.32\textwidth}
        \centering
        \includegraphics[width=\linewidth]{toy/4_epoch_292000_sigma_0.05_train_sigma_0.05.png}
        \caption{$\hat\sigma=0.05$}
    \end{subfigure}
    \begin{subfigure}{0.32\textwidth}
        \centering
        \includegraphics[width=\linewidth]{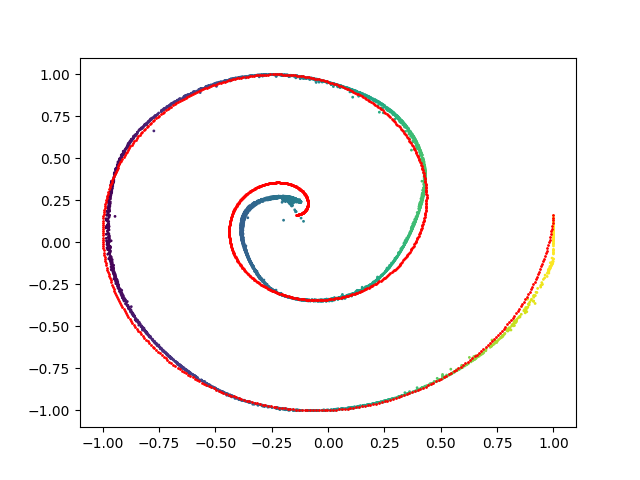}
        \caption{$\hat\sigma=0.1$}
    \end{subfigure}
    \caption{\textbf{A toy example illustrating the impact of tuning 
$\sigma$.} We use a noisy training dataset with $\sigma=0.05$. During pretraining and distillation, we experiment with different values of $\hat \sigma$, representing underestimation, accurate estimation, and overestimation. A slight overestimation of the noise level tends to increase regularization strength, helping the generated data better adhere to the data manifold. }
    \label{fig:sigma_estimate}
\end{figure}

\paragraph{Diverse Corruption Cases}  
Our study primarily focuses on settings where the corruption process involves adding noise. Extending our framework to handle a broader range of corruption operators—such as blurring and downscaling—is an important direction for future research. Addressing these more complex corruption processes could further enhance the general applicability of our method.

\paragraph{Applications in Scientific Discovery}  
Our proposed approach is particularly well-suited for scientific discovery applications, where access to clean observational data is inherently limited. Applying our method to scientific datasets is a promising avenue for future research.

\section{Proofs}

The   result shows that minimizing the Fisher divergence in Eq. \eqref{eq:ideal-score-loss-r} to distill the noisy teacher model induces a distribution that is closer in Wasserstein-2 distance to the clean underlying data distribution. When the support of the underlying data distribution has lower intrinsic dimension, the better our distilled distribution approximates the clean data distribution. We further note that this result focuses on the setting where the underlying generator has low-rank structure. While it is common to make simplifying assumptions on the network architecture to understand score-based models \cite{chen2023score, chen2024learning}, there is also recent work \cite{wang2024diffusion} that has shown when trained on data of low intrinsic dimensionality, score-based models can exhibit low-rank structures. Empirically, we find that neural-network-based distilled models can find such low-dimensional structures through noisy data. An interesting future direction of this work is to understand the influence of neural-network-based parameterizations of the score function along with analyzing the fake score setting. 


Before we dive into the proof, we provide the following lemmas.
\begin{lemma}\label{lemma1}
    [Generalized Woodbury Matrix Identity \cite{higham2002accuracy}]

Given an invertible square matrix \( A \in \mathbb{R}^{n \times n} \), along with matrices \( U \in \mathbb{R}^{n \times k} \) and \( V \in \mathbb{R}^{k \times n} \), define the perturbed matrix: $B = A + U V$.
If \( (I_k + V A^{-1} U) \) is invertible, then the inverse of \( B \) is given by:
\[
B^{-1} = A^{-1} - A^{-1} U (I_k + V A^{-1} U)^{-1} V A^{-1}.
\]
\end{lemma}

\begin{lemma}\label{lemma2}
 The Wasserstein-2 distance between two mean-zero Gaussians $\cN(0,\Sigma_1)$ and $\cN(0,\Sigma_2)$ whose covariance matrices commute, i.e., $\Sigma_1\Sigma_2 = \Sigma_2\Sigma_1$, is given by \begin{align*}  W_2^2(\cN(0,\Sigma_1),\cN(0,\Sigma_2)) 
        = \sum_{i=1}^d \lambda_i(\Sigma_1) + \lambda_i(\Sigma_2) - 2\sqrt{\lambda_i(\Sigma_1)\lambda_i(\Sigma_2)}.
    \end{align*}
\end{lemma}

\begin{lemma}[\cite{mirsky1975trace}]\label{vonNeumann}
    Suppose $A$ and $B$ are $d \times d$ complex matrices with singular values $\sigma_1(A)\geqslant \sigma_2(A)\geqslant\dots\geqslant \sigma_d(A) \geqslant 0$ and $\sigma_1(B)\geqslant \sigma_2(B)\geqslant\dots\geqslant \sigma_d(B) \geqslant 0$, respectively. Then $$|\mathrm{tr}(AB)| \leqslant \sum_{i=1}^d \sigma_i(A)\sigma_i(B).$$
\end{lemma}

\begin{lemma}\label{lem:pca-like-result}
    Let $E \in \mathbb{R}^{d \times r}$ with $r < d$ have orthonormal columns and $\Sigma \in \mathbb{R}^{r \times r}$ be symmetric positive definite. Then \begin{align*}
        \argmax_{U^TU = I_r} \mathrm{tr}(EE^TU\Sigma U^T) = \{EQ : Q\ \text{orthogonal}\}.
    \end{align*}
\end{lemma}
\begin{proof}[Proof of Lemma \ref{lem:pca-like-result}]
    Observe that by the von Neumann trace inequality (Lemma \ref{vonNeumann}), we have that for any feasible $U$, $$\mathrm{tr}(EE^TU\Sigma U^T) =\mathrm{tr}(U^TEE^TU\Sigma)  \leqslant \sum_{i=1}^r \lambda_i(U^TEE^TU)\lambda_i(\Sigma) = \sum_{i=1}^r \lambda_i(EE^T)\lambda_i(\Sigma) = \sum_{i=1}^r \lambda_i(\Sigma).$$ Hence, to maximize $U \mapsto \mathrm{tr}(EE^TU\Sigma U^T)$ over $\{U : U^TU = I_r\}$, we want $U^*$ to satisfy $\mathrm{tr}(EE^TU^*\Sigma (U^*)^T) = \sum_{i=1}^r \lambda_i(\Sigma).$ 
    
    We claim that this occurs if and only if $U^* = EQ$ for some orthogonal $Q$. If $U^*= EQ$, then $(U^*)^TEE^TU^* = Q^TE^TEE^TEQ = I$ so $$\mathrm{tr}(EE^TU^*\Sigma (U^*)^T) = \mathrm{tr}((U^*)^TEE^TU^*\Sigma) = \mathrm{tr}(\Sigma) = \sum_{i=1}^r \lambda_i(\Sigma).$$ For the other direction, suppose $U^*$ maximizes the objective. Then $$\mathrm{tr}((U^*)^TEE^TU^*\Sigma) = \mathrm{tr}(\Sigma) \Longleftrightarrow \mathrm{tr}\left(((U^*)^TEE^TU^* - I_r)\Sigma\right) = 0.$$ Set $Q := E^TU^*.$ Note that the eigenvalues of $Q^TQ$ are bounded by $1$ so $Q^TQ - I_r$ is negative semi-definite while $\Sigma$ is positive definite. But if $\mathrm{tr}((Q^TQ-I_r)\Sigma) = 0$, by positive definiteness of $\Sigma$, we must have $Q^TQ - I_r = 0$, i.e., $Q^TQ = I_r$. This means $Q$ is orthogonal. Since $Q$ is orthogonal and $Q= E^TU^* \Longrightarrow U^* = EQ$, as desired. 
\end{proof}

    

\begin{lemma}\label{lem:eigenvalue-time-dependent-function}
    Fix $\sigma > 0$ and consider a noise schedule $\sigma_t > 0$ for $t \in (0,1)$ such that $(\sigma_t) \subseteq [\sigma_{\min}, \sigma_{\max}]$ for some $0 < \sigma_{\min} < \sigma_{\max} < \infty$. Define the function $f_{\sigma} : (0,\infty) \rightarrow \mathbb{R}$ by $$f_{\sigma}(u) := \mathbb{E}_{t \sim \mathrm{Unif}(0,1)}\left[\frac{u}{(\sigma^2+\sigma_t^2+1)^2} - \frac{u}{\sigma_t^2(u+\sigma_t^{2})}\right].$$ Then $f_{\sigma}$ is strictly convex and has a unique minimizer at $u^* = \sigma^2 + 1$ which is the unique solution to the equation $$\mathbb{E}_{t \sim \mathrm{Unif}(0,1)}\left[\frac{1}{(\sigma^2+\sigma_t^2+1)^2}\right] = \mathbb{E}_{t\sim\mathrm{Unif}(0,1)}\left[\frac{1}{(u^*+\sigma_t^2)^2}\right].$$
\end{lemma}
\begin{proof}[Proof of Lemma \ref{lem:eigenvalue-time-dependent-function}] First, note that the conditions on $\sigma_t$ ensure that all of the following expectations are finite. By direct calculation, we have the derivatives of $f_{\sigma}$ are \begin{align*}
    f'_{\sigma}(u) & = \mathbb{E}_t\left[\frac{1}{(\sigma^2 +\sigma_t^2+ 1)^2}\right] - \mathbb{E}_t\left[\frac{1}{(\sigma_t^2 + u)^2}\right] \ \text{and}\ 
    f''_{\sigma}(u)  = \mathbb{E}_t\left[\frac{2}{(\sigma_t^2+u)^3}\right]. 
\end{align*} Hence $f''_{\sigma}(u) > 0$ for all $u > 0$ so $f_{\sigma}$ is strictly convex. To find its minimizer $u^*$, setting the derivative equal to $0$ yields $u^*$ must satisfy $$\mathbb{E}_t\left[\frac{1}{(\sigma^2 +\sigma_t^2+ 1)^2}\right] = \mathbb{E}_t\left[\frac{1}{(\sigma_t^2 + u^*)^2}\right].$$ Note that the point $u^* = 1+\sigma^2$ clearly satisfies the critical point equation. Uniqueness follows due to strict convexity. 
    
\end{proof}

\subsection{Proof of Theorem \ref{thm:time-dependent-theorem}} \label{sec:appx-theorem-proof}

We break down the proof of Theorem \ref{thm:time-dependent-theorem} into three key steps. First, we show that minimizing the objective (Eq. \eqref{eq:ideal-score-loss-r}) is equivalent to minimizing a simpler objective. Then, we show that we can derive exact analytical expressions for the global minimizers of this simpler objective, which are then global minimizers of the original score-based loss. Finally, we will directly compute the Wasserstein distance between our learned distilled distribution to the clean distribution and compare this to the noisy distribution.

\paragraph{Reduction of objective function:} For $\sigma_t > 0$, define $p^{\sigma_t}_{G_{\theta}} := \mathcal{N}(0,UV^TVU^T + \sigma^2_t I_d)$ and $s_{\sigma,\sigma_t}(x):= -(EE^T + (\sigma^2 + \sigma_t^2) I_d)^{-1}x$. For the proof, we will assume our parameters $\theta = (U,V) \in \Theta$ so that $U^TU = I_r$ and $V^TV \succ 0$. We consider minimizing the loss $$\mathcal{L}(\theta) := \mathbb{E}_{t \sim \mathrm{Unif}(0,1)}\mathbb{E}_{x_t \sim p^{\sigma_t}_{G_{\theta}}}\left[\left\|s_{\sigma,\sigma_t}(x_t) - \nabla \log p^{\sigma_t}_{G_{\theta}}(x_t) \right\|_2^2\right].$$ For $t \in (0,1)$, consider the inner expectation of the loss $$\Tilde{\mathcal{L}}_{t}(\theta) : = \mathbb{E}_{x_t \sim p^{\sigma_t}_{G_{\theta}}}\left[\left\|s_{\sigma,\sigma_t}(x_t) - \nabla \log p^{\sigma_t}_{G_{\theta}}(x_t) \right\|_2^2\right].$$ For notational convenience, set $\Sigma_{\sigma,t} := EE^T + (\sigma^2 + \sigma_t^2) I_d$ and $\Sigma_{\theta,t} := UV^TVU^T + \sigma^2_t I_d.$ Then $s_{\sigma,\sigma_t}(x) := -\Sigma_{\sigma,t}^{-1}x$ and $\nabla \log p^{\sigma_t}_{G_{\theta}}(x) := -\Sigma_{\theta,t}^{-1}x.$ First, recall that for $x_t \sim p^{\sigma_t}_{G_{\theta}}$ and any matrix $\Sigma$, $\mathbb{E}_{x \sim p^{\sigma_t}_{G_{\theta}}}[\|\Sigma x_t\|_2^2] = \|\Sigma\Sigma_{\theta,t}^{1/2}\|_F^2.$ Using this, we can compute the loss as follows: \begin{align*}
    \tilde{\mathcal{L}}_{t}(\theta) & = \mathbb{E}_{x_t \sim p^{\sigma_t}_{G_{\theta}}}\left[\|(\Sigma_{\sigma,t}^{-1} - \Sigma_{\theta,t}^{-1})x_t\|_2^2\right] \\
    & = \|(\Sigma_{\sigma,t}^{-1} - \Sigma_{\theta,t}^{-1})\Sigma_{\theta,t}^{1/2}\|_F^2 \\
    & = \mathrm{tr}\left(\Sigma_{\theta,t}^{1/2}(\Sigma_{\sigma,t}^{-1} - \Sigma_{\theta,t}^{-1}) (\Sigma_{\sigma,t}^{-1} - \Sigma_{\theta,t}^{-1})\Sigma_{\theta,t}^{1/2}\right) \\
    & = \mathrm{tr}\left(\Sigma_{\theta,t}(\Sigma_{\sigma,t}^{-1} - \Sigma_{\theta,t}^{-1}) (\Sigma_{\sigma,t}^{-1} - \Sigma_{\theta,t}^{-1})\right)\\
    & = \mathrm{tr}\left((\Sigma_{\theta,t}\Sigma_{\sigma,t}^{-1} - I_d) (\Sigma_{\sigma,t}^{-1} - \Sigma_{\theta,t}^{-1})\right) \\
    & = \mathrm{tr}\left(\Sigma_{\theta,t}\Sigma_{\sigma,t}^{-2} -\Sigma_{\theta,t}\Sigma_{\sigma,t}^{-1}\Sigma_{\theta,t}^{-1} - \Sigma_{\sigma,t}^{-1} + \Sigma_{\theta,t}^{-1}\right) \\
    & = \mathrm{tr}\left(\Sigma_{\theta,t}\Sigma_{\sigma,t}^{-2}\right) -\mathrm{tr}\left(\Sigma_{\theta,t}\Sigma_{\sigma,t}^{-1}\Sigma_{\theta,t}^{-1}\right) - \mathrm{tr}\left(\Sigma_{\sigma,t}^{-1}\right) + \mathrm{tr}\left(\Sigma_{\theta,t}^{-1}\right) \\
    & = \mathrm{tr}\left(\Sigma_{\sigma,t}^{-2}\Sigma_{\theta,t}\right) - 2\mathrm{tr}\left(\Sigma_{\sigma,t}^{-1}\right) + \mathrm{tr}\left(\Sigma_{\theta,t}^{-1}\right) \\
    & =: C_{\sigma,t} + \mathrm{tr}\left(\Sigma_{\sigma,t}^{-2}\Sigma_{\theta,t}\right) +\mathrm{tr}\left(\Sigma_{\theta,t}^{-1}\right).
\end{align*} Using Lemma \ref{lemma2}, it is straightforward to see that \begin{align*}
    \Sigma_{\sigma,t}^{-1} & = \frac{1}{\sigma^2 + \sigma_t^2} I_d - \frac{1}{(\sigma^2 + \sigma_t^2)^2(\sigma^2 + \sigma_t^2+1)}EE^T\ \text{and}\ \\
    \Sigma_{\theta,t}^{-1} & = \sigma_t^{-2}I_d - \sigma_t^{-4}U\left((V^TV)^{-1} + \sigma_t^{-2}I_r\right)^{-1}U^T
\end{align*} Hence the third term in $\tilde{\mathcal{L}}_{t}$ is given by $$\mathrm{tr}(\Sigma_{\theta,t}^{-1}) = \mathrm{tr}\left(\sigma_t^{-2}I_d - \sigma_t^{-4}U\left((V^TV)^{-1} + \sigma_t^{-2}I_r\right)^{-1}U^T\right) =: C_{\sigma_t} - \sigma_t^{-4}\mathrm{tr}\left(\left((V^TV)^{-1} + \sigma_t^{-2}I_r\right)^{-1}\right)$$ where we used the cyclic property of the trace and $U^TU=I_r$ in the last equality. For the second term, let $\beta_t^2 := \sigma^2 + \sigma_t^2$ and $\gamma_{\sigma,t}:= \frac{1}{\beta_t^2(\beta_t^2+1)}$. Then we have by direct computation, \begin{align*}
    \mathrm{tr}\left(\Sigma_{\sigma,t}^{-2}\Sigma_{\theta,t}\right) & = \mathrm{tr}\left(\left(\beta_t^{-2}I_d - \gamma_{\sigma,t}EE^T\right)\left(\beta_t^{-2}I_d - \gamma_{\sigma,t}EE^T\right) (UV^TVU^T+\sigma_t^2I_d)\right) \\
    & = \mathrm{tr}\left(\left(\beta_t^{-4}I_d - 2\beta_t^{-2}\gamma_{\sigma,t}EE^T + \gamma_{\sigma,t}^2EE^T\right) (UV^TVU^T+\sigma_t^2I_d)\right) \\
    & = \mathrm{tr}\left(\beta_t^{-4}UV^TVU^T - \sigma_t^2\beta_t^{-4}I_d +\left(\gamma_{\sigma,t}^2 - 2\beta_t^{-2}\gamma_{\sigma,t}\right) EE^TUV^TVU^T\right) \\
    & \qquad - \mathrm{tr}\left(2\beta_t^{-2}\sigma_t^2EE^T + \gamma_{\sigma,t}^2\sigma_t^2I_d\right) \\
    & =:\Tilde{C}_{\sigma,t} + \beta_t^{-4}\mathrm{tr}(UV^TVU^T) +\left(\gamma_{\sigma,t}^2 - 2\beta_t^{-2}\gamma_{\sigma,t}\right) \cdot \mathrm{tr}(EE^TUV^TVU^T) \\
    & = \Tilde{C}_{\sigma,t} + \beta_t^{-4}\mathrm{tr}(V^TV) +\left(\gamma_{\sigma,t}^2 - 2\beta_t^{-2}\gamma_{\sigma,t}\right) \cdot \mathrm{tr}(EE^TUV^TVU^T)
    \end{align*} where we used the cyclic property of trace and orthogonality of $U$ in the final line. Combining the above displays, we get that there exists a constant $C_{\sigma,\sigma_t} := C_{\sigma,t} + C_{\sigma_t} + \Tilde{C}_{\sigma,t}$ such that \begin{align*}
        \tilde{\mathcal{L}}_{t}(\theta) & = C_{\sigma,\sigma_t} + \left(\frac{1}{\beta_t^4(\beta_t^2+1)^2} - \frac{2}{\beta_t^4(\beta_t^2+1)}\right)\cdot\mathrm{tr}(EE^TUV^TVU^T) \\
        & + \beta_t^{-4}\mathrm{tr}(V^TV) - \sigma_t^{-4}\mathrm{tr}\left(\left((V^TV)^{-1} + \sigma_t^{-2}I_r\right)^{-1}\right) \\
        & =: C_{\sigma,\sigma_t} + B_t(U,V) + R_t(V)
    \end{align*} where we have defined the quantities \begin{align*}
        B_t(U,V) & := \left(\frac{1}{\beta_t^4(\beta_t^2+1)^2} - \frac{2}{\beta_t^4(\beta_t^2+1)}\right)\cdot\mathrm{tr}(EE^TUV^TVU^T)\ \text{and} \\
        R_t(V) & := \beta_t^{-4}\mathrm{tr}(V^TV) -\sigma_t^{-4}\mathrm{tr}\left(\left((V^TV)^{-1} + \sigma_t^{-2}I_r\right)^{-1}\right).
    \end{align*} Recalling the definition of $\mathcal{L}(\cdot)$, we have that \begin{align*}
        \mathcal{L}(\theta) = \mathbb{E}_{t \sim \mathrm{Unif}(0,1)}\left[\tilde{\mathcal{L}}_t(\theta)\right] =  \mathbb{E}_{t \sim \mathrm{Unif}(0,1)}\left[C_{\sigma,\sigma_t} +B_t(U,V) + R_t(U,V)\right].
    \end{align*} Hence we have the equivalence $$\argmin_{\theta \in \Theta} \mathcal{L}(\theta) = \argmin_{\theta \in \Theta} \mathbb{E}_{t \sim \mathrm{Unif}(0,1)}\left[B_t(U,V)\right] + \mathbb{E}_{t \sim \mathrm{Unif}(0,1)}[R_t(V)].$$

    \paragraph{Form of minimizers:} We use the shorthand notation $\mathbb{E}_t[\cdot] := \mathbb{E}_{t \sim \mathrm{Unif}(0,1)}[\cdot]$. First, note that we can first minimize $\mathbb{E}_t[B_t(U,V)]$ over feasible $U$. But note that $$\mathbb{E}_{t}[B_t(U,V)] = \underbrace{\mathbb{E}_t \left[\frac{1}{\beta_t^4(\beta_t^2 + 1)^2} - \frac{2}{\beta_t^4(\beta_t^2 + 1)}\right]}_{< 0} \mathrm{tr}(EE^TUV^TVU^T)$$ since for any $t$, $\frac{1}{(\beta_t^2 + 1)^2} < \frac{2}{(\beta_t^2 + 1)}.$ Hence minimizing $\mathbb{E}_t[B_t(U,V)]$ is equivalent to maximizing $\mathrm{tr}(EE^TUV^TVU^T).$ Taking $\Sigma = V^TV$ in Lemma \ref{lem:pca-like-result}, we have that the minimizer of $\mathbb{E}_t[B_t(U,V)]$ is given by $$U^* = EQ\ \text{for some orthogonal}\ Q.$$ Moreover, the proof of Lemma \ref{lem:pca-like-result} shows that $\mathrm{tr}(EE^TU^*V^TV(U^*)^T) = \mathrm{tr}(V^TV)$. This gives
    \begin{align*}
        \mathbb{E}_t[B_t(U^*,V)] & = \mathbb{E}_t\left(\frac{1}{\beta_t^4(\beta_t^2+1)^2} - \frac{2}{\beta_t^4(\beta_t^2+1)}\right) \mathrm{tr}(V^TV).
    \end{align*} In summary, we now must minimize the following with respect to invertible $V$: \begin{align*}
        \mathbb{E}_t[B_t(U^*,V)] + \mathbb{E}_t[R_t(V)] & = \mathbb{E}_t\left(\frac{1}{\beta_t^4(\beta_t^2+1)^2} - \frac{2}{\beta_t^4(\beta_t^2+1)} + \frac{1}{\beta_t^4}\right) \mathrm{tr}(V^TV) \\
        & \qquad - \mathbb{E}_t\left[\sigma_t^{-4}\mathrm{tr}\left(\left((V^TV)^{-1} + \sigma_t^{-2}I_r\right)^{-1}\right)\right] \\
        & = \mathbb{E}_t\left(\frac{1}{\beta_t^4}\left(\frac{1}{\beta_t^2 + 1} - 1\right)^2\right) \mathrm{tr}(V^TV) - \mathbb{E}_t\left[\sigma_t^{-4}\mathrm{tr}\left(\left((V^TV)^{-1} + \sigma_t^{-2}I_r\right)^{-1}\right)\right] \\
        & = \mathbb{E}_t\left(\frac{1}{\beta_t^4}\left(\frac{\beta_t^2}{\beta_t^2 + 1}\right)^2\right) \mathrm{tr}(V^TV) - - \mathbb{E}_t\left[\sigma_t^{-4}\mathrm{tr}\left(\left((V^TV)^{-1} + \sigma_t^{-2}I_r\right)^{-1}\right)\right] \\
        & = \mathbb{E}_t\left(\frac{1}{(\beta_t^2 + 1)^2}\right) \mathrm{tr}(V^TV) - \mathbb{E}_t\left[\sigma_t^{-4}\mathrm{tr}\left(\left((V^TV)^{-1} + \sigma_t^{-2}I_r\right)^{-1}\right)\right]
    \end{align*} where in the second equality, we completed the square.

    We now claim that $\mathbb{E}_t[B_t(U^*,V)] + \mathbb{E}_t[R_t(V)]$ solely depends on the eigenvalues of $V^TV$. In particular, for invertible $V$, note that $V^TV \succ 0$ so it admits the decomposition $V^TV = Q\Lambda Q^T$ where $Q^TQ=QQ^T=I_r$ and $\Lambda$ is a diagonal matrix with positive entries $\Lambda_{ii}= \lambda_i(V^TV) > 0$. Hence $\mathrm{tr}(V^TV) = \mathrm{tr}(Q\Lambda Q^T) = \mathrm{tr}(Q^TQ\Lambda)= \mathrm{tr}(\Lambda) = \sum_{i=1}^r \lambda_i(V^TV).$ Likewise, we have using the orthogonality of $Q$ that for any $\varepsilon > 0$, \begin{align*}
        \mathrm{tr}\left(\left((V^TV)^{-1} + \varepsilon^{-2}I_r\right)^{-1}\right) & = \mathrm{tr}\left(\left((Q\Lambda Q^T)^{-1} + \varepsilon^{-2}I_r\right)^{-1}\right)\\
        & = \mathrm{tr}\left(\left(Q\Lambda ^{-1}Q^T + \varepsilon^{-2}QQ^T\right)^{-1}\right) \\
        & = \mathrm{tr}\left(\left(Q\left(\Lambda^{-1} + \varepsilon^{-2}I_r\right)Q^T\right)^{-1}\right) \\
        & = \mathrm{tr}\left(Q\left(\Lambda^{-1} + \varepsilon^{-2}I_r\right)^{-1}Q^T\right) \\
        & = \mathrm{tr}\left(\left(\Lambda^{-1} + \varepsilon^{-2}I_r\right)^{-1}\right) \\
        & = \sum_{i=1}^r \frac{1}{\lambda_i(V^TV)^{-1} + \varepsilon^{-2}} \\
        & = \sum_{i=1}^r \frac{\lambda_i(V^TV) \cdot \varepsilon^2}{\lambda_i(V^TV) + \varepsilon^2}.
    \end{align*} In sum, the final objective is a particular function of the eigenvalues of $V^TV$: \begin{align*}
        \mathbb{E}_t[B_t(U^*,V)] + \mathbb{E}_t[R_t(V)] & = \sum_{i=1}^r \mathbb{E}_t\left[ \frac{\lambda_i(V^TV)}{(\beta_t^2+1)^2} - \frac{\lambda_i(V^TV)}{\sigma_t^2(\lambda_i(V^TV) + \sigma_t^{2})} \right] \\
        & = \sum_{i=1}^r \mathbb{E}_t\left[ \frac{\lambda_i(V^TV)}{(\sigma^2 + \sigma_t^2+1)^2} - \frac{\lambda_i(V^TV)}{\sigma_t^2(\lambda_i(V^TV) + \sigma_t^{2})} \right] \\
        & =: \sum_{i=1}^r f_{\sigma}(\lambda_i(V^TV)).
    \end{align*} In Lemma \ref{lem:eigenvalue-time-dependent-function}, we show that the function $u \mapsto f_{\sigma}(u)$ is strictly convex on $(0,\infty)$ with a unique minimizer at $1+\sigma^2$. Thus $V \mapsto B(U^*,V) + R(V)$ for invertible $V$ is minimized when the gram matrix of $V^*_{\sigma}$ has equal eigenvalues $\lambda_i((V^*_{\sigma})^TV^*_{\sigma}) = 1+\sigma^2$ for all $i \in [r]$. Since all of its eigenvalues are the same, by the Spectral Theorem, we must have that  $(V^*_{\sigma})^TV^*_{\sigma} = (1+\sigma^2 )I_r.$

        \paragraph{Wasserstein bound:} We now show the Wasserstein error bound. Note that $\theta^*_{\sigma} = (U^*,V_{\sigma}^*)$ induces the distribution $p_{G_{\theta^*_{\sigma}}}$ defined by $$x = G_{\theta^*_{\sigma}}(z),\ z \sim \mathcal{N}(0,I_d) \Longleftrightarrow x \sim p_{G_{\theta^*_{\sigma}}} :=  \mathcal{N}(0,EQ(V_{\sigma}^*)^TV^*_{\sigma}Q^TE^T) = \mathcal{N}(0,(1+\sigma^2)EE^T).$$ Then by Lemma \ref{lemma2}, we have
    \[
\begin{aligned}
W_2^2(p_{Y,\sigma},p_X) &= r\Bigl(1+\sigma^2+1-2\sqrt{1+\sigma^2}\Bigr) + (d-r)\sigma^2,\\
W_2^2(p_{G_{\theta^*_{\sigma}}},p_X) &= r\Bigl(1+\sigma^2+1-2\sqrt{1+\sigma^2}\Bigr).
\end{aligned}
\]  This gives $$W_2^2(p_{G_{\theta^*_{\sigma}}},p_X) = W_2^2(p_{Y,\sigma},p_X) - (d-r)\sigma^2 < W_2^2(p_{Y,\sigma},p_X).$$

\section{Distillation Loss}
\label{app:distillation}
In Section~\ref{sec:method} and Algorithm~\ref{alg:ambient_distill}, we introduced the generator loss formulation given by Eq.~\ref{eq:distillation}. However, since Eq.~\ref{eq:distillation} cannot be directly utilized for training the generator, it requires instantiation. As discussed in Section~\ref{sec:experiment}, we adopt three widely used distillation methods: SDS~\cite{poole2022dreamfusion}, DMD~\cite{yin2024one} (also referred to as Diff-Instruct~\cite{luo2023diff} or VSD~\cite{wang2024prolificdreamer}), and SiD~\cite{zhou2024score}. For completeness, we present their corresponding generator loss formulations below, while deferring implementation details such as scheduling and hyperparameter selection to the original papers.

We define the perturbed sample as 
\begin{equation}
    x_t = x_g + \sigma_t \epsilon, \quad \epsilon \sim \mathcal{N}(0, I_d).
\end{equation}
By the results in~\cite{song2021scorebased, karras2022elucidating}, the score function, mean prediction function, and epsilon prediction function are related as follows:
\begin{equation}
    s_\phi(x_t, t) = -\frac{x_t - f_\phi(x_t, t)}{\sigma_t^2}, \quad
    \varepsilon_\phi(x_t, t) = -\sigma_t s_\phi(x_t, t).
\end{equation}
These relationships also extend to the generative process of the fake diffusion model $f_\psi$, ensuring consistency across different parametrizations.

The gradient of the generator loss for SDS is given by:
\begin{equation}
    \nabla_\theta\mathcal{L}_{\text{SDS}} = \mathbb{E}_{z,t,x,x_t,\epsilon} 
    \left[ w_t (\varepsilon_\phi(x_t,t)-\epsilon) \frac{dG}{d\theta} \right].
    \label{eq:sds_generator}
\end{equation}
Note that we don't include the training of fake diffusion in the D-SDS to follow the original SDS paper.

For DMD, the generator loss gradient takes the form:
\begin{equation}
    \nabla_\theta\mathcal{L}_{\text{DMD}} = \mathbb{E}_{z,t,x,x_t,\epsilon} 
    \left[ w_t (s_{\psi}(x_t,t) - s_{\phi}(x_t,t)) \frac{dG}{d\theta} \right].
    \label{eq:dmd_generator}
\end{equation}

The SiD loss function gradient is formulated as:
\begin{align}
    \nabla_\theta\mathcal{L}_{\text{SiD}} &= 
        \nabla_\theta\mathbb{E}_{z,t,x,x_t,\epsilon} \Big[ (1 - \alpha) w(t) \| f_{\psi}(x_t, t) - f_{\phi}(x_t, t) \|_2^2 \nonumber\\ 
        &\quad + w(t) \left( f_{\phi}(x_t, t) - f_{\psi}(x_t, t) \right)^T \left( f_{\psi}(x_t, t) - x_g \right) \Big].
        \label{eq:sid_generator}
\end{align}

These formulations encapsulate the key differences in how D-SDS, D-DMD, and D-SiD approach generator training, each leveraging different mechanisms to refine the learned score or mean function.

\section{Implementation Details}\label{app:implementation}

\subsection{Denoising Training, Sampling, and Distillation Algorithm}
\label{app:alg}

In this section, we provide a comprehensive details of the denoising training, distillation, and sampling procedures. Algorithm~\ref{alg:ambient_training} details the application of Ambient Tweedie (Eq \eqref{eq:ambient_diffusion}) for pertaining with the adjusted diffusion objective. Algorithm~\ref{alg:ambient_sampling} outlines the sampling procedure employed to obtain the results for Ambient-Truncate and Ambient-Full, as discussed in Section~\ref{sec:experiment}.

\begin{algorithm}
\caption{Pretraining with Adjusted Diffusion Objectives \cite{daras2024ambienttweedie}}
\label{alg:ambient_training}
\begin{algorithmic}[1]
\Procedure{Denoising-Pretraining} {$\{y^{(i)}\}_{i=1}^N$, $\sigma$, $p(\sigma_t),K$} 
\Comment{Diffusion schedule $p(\sigma_t)$.}

\For{k=1,\dots,K} \Comment{Training K iterations.}
    \State Sample batch $y \sim \{y^{(i)}\}_{i=1}^N$, $\sigma_t \sim p(\sigma_t)$, $\epsilon \sim \mathcal{N}(0, I_d)$
    \State $\sigma_t = \max\{\sigma, \sigma_t\}$ \Comment{Noise level clip.}
    \State $x_t = y + \sqrt{\sigma_t^2 - \sigma^2} \cdot \epsilon$ \Comment{Inject noise to $\sigma_t$.}
    \State Update parameters of $f$ by gradient descent step \Comment{Eq \ref{eq:ambient_diffusion}}.
    \[
    \nabla \left\| \frac{\sigma_t^2 - \sigma^2}{\sigma_t^2} f(x_t, t) + \frac{\sigma^2}{\sigma_t^2} x_t - y \right\|^2
    \] 
\EndFor

\State \Return Trained diffusion model $f$  

\EndProcedure

\end{algorithmic}
\end{algorithm}

\noindent
\begin{algorithm}
\caption{Ambient Sampling \cite{daras2024ambienttweedie}}
\label{alg:ambient_sampling}
\begin{algorithmic}[1]
\Procedure{Ambient-Sampling}{$f, \sigma, \{\sigma_t\}_{t=0}^T$}
    \State Sample $x_T \sim \mathcal{N}(0, \sigma_T I_d)$
    \For{$t = T, T-1, \dots, 1$}
        \State $\hat{x}_0 \leftarrow f(x_t, t)$
        \If{Truncation is applied and $\sigma_{t-1} < \sigma$} 
            \State \textbf{return} $\hat{x}_0$ \hfill \Comment{Truncated Sampling}
        \EndIf
        \State $x_{t-1} \leftarrow x_t - \frac{\sigma_t - \sigma_{t-1}}{\sigma_t} (x_t - \hat{x}_0)$
    \EndFor
    \State \textbf{return} $x_0$ \hfill \Comment{Full Sampling}
\EndProcedure
\end{algorithmic}
\end{algorithm}


\subsection{Training and Distillation Details and Hyperparameter Selection}
\label{app:hyper}
For training the teacher diffusion model, we train on 200 million images for CIFAR-10, matching the computational budget of EDM, while all other datasets are trained on 100 million images, corresponding to half of the EDM computational budget. Inference wall time is measured using four A6000 GPUs with a batch size of 1024. All images are normalized to the range $[-1,1]$ before adding additive Gaussian noise. We adopt the training hyperparameters from the EDM codebase~\cite{karras2022elucidating}. 

For distillation, we train CIFAR-10 on 100 million images and all other datasets on 15 million images, as we observe that this training budget is sufficient to achieve a competitive FID score. All the  hyperparameters remain identical to those in SiD~\cite{zhou2024score}.

For CelebA-HQ, the setting is the same as FFHQ and AFHQ-v2 except that the dropout rate is 0.15.
For our experiments with consistency, we use 8 steps for the reverse sampling and 32 samples to estimate the expectations. We use a fixed coefficient to weight the consistency loss that is chosen as a hyperparameter from the set of $\{ 0.1, 1.0, 10.0 \}$ to maximize performance. Upon acceptance of this work, we will provide all the code and checkpoints to accelerate research in this area.

\subsection{Evaluation}

We generate 50,000 images to compute FID. Each FID number reported in this paper is the average of three independent FID computations that correspond to the seeds: 0-49,999, 50,000-99,999, 100,000-149999.
  \begin{table}[h]
    \centering
     \caption{\textbf{Results of our methods (D-SDS, D-DMD, D-SiD) on CIFAR-10 at various noise levels with standard deviations.}   }
 {%
    \begin{tabular}{l|ccc} 
    \toprule
        \textbf{Methods} &  \textbf{\boldmath$\sigma=0.1$} & \textbf{\boldmath$\sigma=0.2$} & \textbf{\boldmath$\sigma=0.4$} \\ 
    \midrule 
  {D-SDS} & > 200 & > 200 & > 200 \\
    {D-DMD}  & 12.52 {\small$\pm$0.04} & 7.48 {\small$\pm$0.06} & 30.09 {\small$\pm$0.23} \\
  {D-SiD} &  \textbf{3.98} {\small$\pm$0.04} & \textbf{4.77} {\small$\pm$0.03} & \textbf{21.63} {\small$\pm$0.03} \\
    \bottomrule
    \end{tabular}%
    }
\end{table}

\begin{table}[h]
    \centering
        \caption{\textbf{Results of our method on various datasets including CIFAR-10, FFHQ, CelebA-HQ and AFHQ-v2 at $\sigma=0.2$ with standard deviations.}  }
    \begin{tabular}{l|cccc}
    \toprule
        \textbf{Methods}  &  \textbf{CIFAR-10} & \textbf{FFHQ} & \textbf{CelebA-HQ} & \textbf{AFHQ-v2} \\ \midrule
        Observation & 127.22 {\small$\pm$0.12} & 110.83 {\small$\pm$0.22} & 107.22 {\small$\pm$0.18}  & 51.51 {\small$\pm$0.15}  \\  
        Ambient-Full & 60.73 {\small$\pm$0.21} & 41.52 {\small$\pm$0.10} & 61.14 {\small$\pm$0.14} &   17.93 {\small$\pm$0.03} \\   
        Ambient-Truncated & 12.21 {\small$\pm$0.03} & 14.67 {\small$\pm$0.02} & 13.90 {\small$\pm$0.01} & 9.82 {\small$\pm$0.02} \\  \midrule
        DSD (one-step) & \textbf{4.77} {\small$\pm$0.03} & \textbf{6.29} {\small $\pm$0.15} & \textbf{6.48} {\small $\pm$0.09} & \textbf{5.42} {\small $\pm$0.08}\\ \bottomrule
    \end{tabular} 
\end{table}

\subsection{Model Selection Criterion}
\label{app:proximal}
  
We present the evolution of FID and Proximal FID results on FFHQ, CelebA-HQ, and AFHQ in Fig. \ref{fig:3dataset}.

  \begin{figure}[htbp]
    \centering
    \begin{subfigure}{0.32\textwidth}
        \centering
        \includegraphics[width=\linewidth]{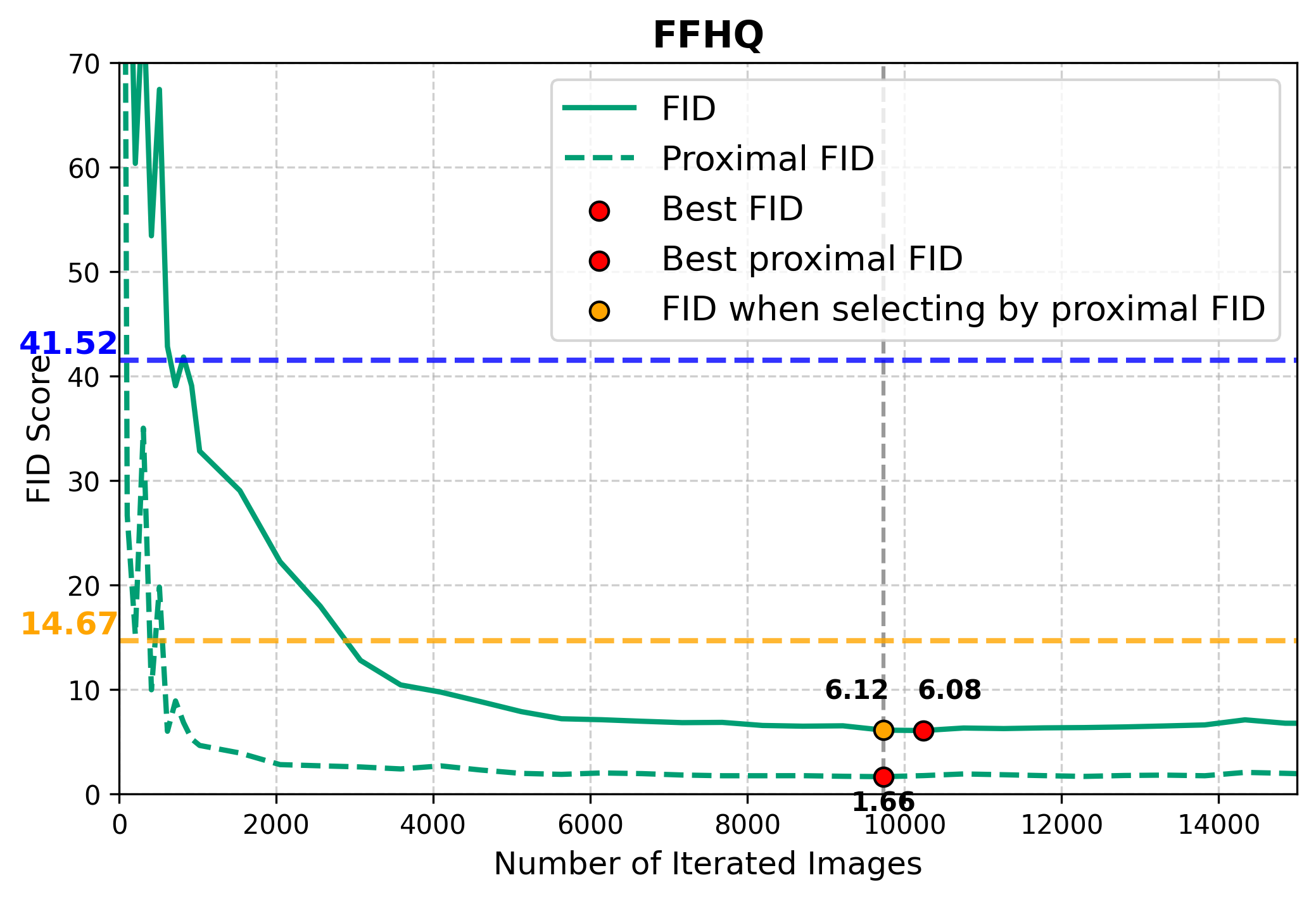}
    \end{subfigure}
    \begin{subfigure}{0.32\textwidth}
        \centering
        \includegraphics[width=\linewidth]{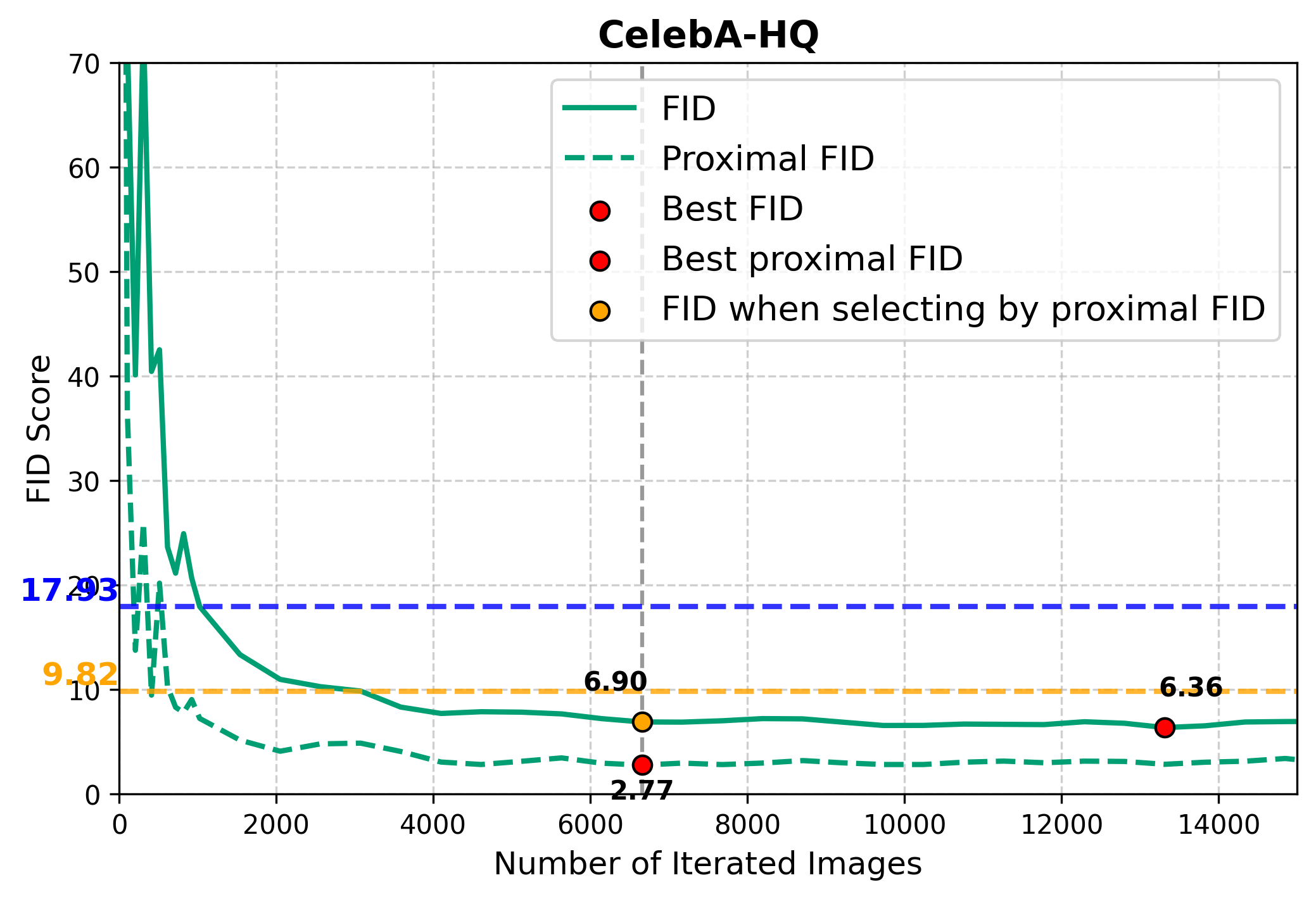}
    \end{subfigure}
    \begin{subfigure}{0.32\textwidth}
        \centering
        \includegraphics[width=\linewidth]{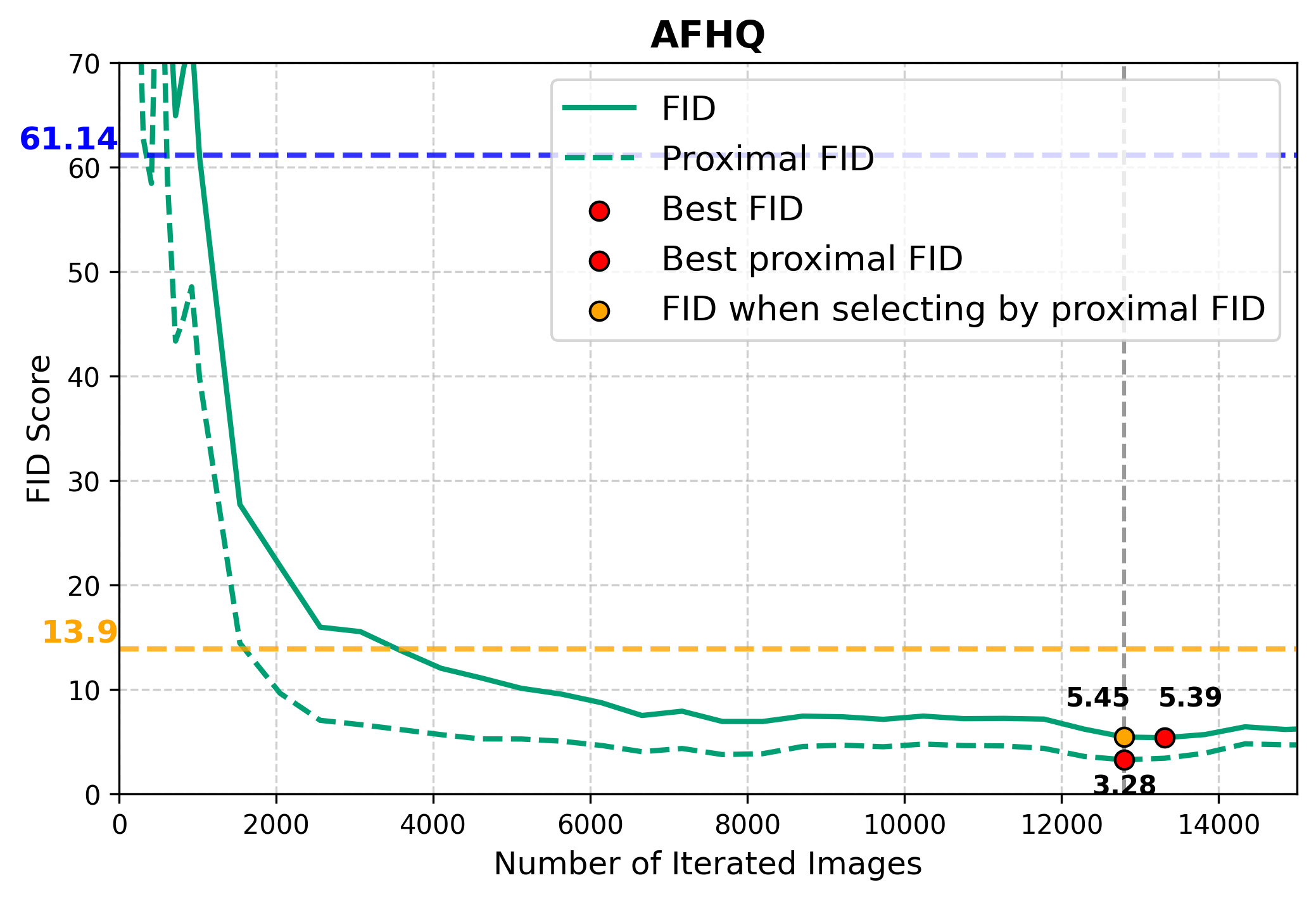}
    \end{subfigure}
    \caption{\textbf{Evolution of FID and Proximal FID results on FFHQ , CelebAHQ and AFHQ-v2.}  Proximal FID serves as a reliable alternative to true FID, consistently selecting models whose ground-truth FID is close to the best achievable FID.}
    \label{fig:3dataset}
\end{figure}

\subsection{Training and Inference Efficiency}

We demonstrated that our proposed methods not only improve performance metrics but also enhance the overall efficiency of both the training and inference phases, as briefly illustrated in Fig. \ref{fig:fid_pretrain_distill}. A detailed time analysis is provided in Tab. \ref{tab:efficiency}. 
During training, the additional distillation phase introduces only a minor overhead, as FID decreases rapidly and surpasses the teacher diffusion model, {Ambient-Truncated}, within just \textbf{4 hours}.  
For inference, our one-step generator enables the generation of 50k images in only 20 seconds—compared to 10 minutes with the diffusion model—achieving a \textbf{30$\times$ speedup}.  Inference wall time is recorded using a batch size of 1024 on 4 Nvidia
RTX A6000 GPUs.
These results confirm that \textbf{denoising score distillation is not merely a trade-off between quality and speed but a mechanism for improving both simultaneously}. Our findings challenge conventional perspectives on distillation and suggest distillation as a new direction for learning generative models from corrupted data.




\begin{table}[!t]
    \centering
   \caption{\textbf{Training and inference efficiency of our method.}  During training, the additional distillation phase introduces only a minor overhead, as FID decreases rapidly and surpasses the teacher diffusion model, {Ambient-Truncated}, within just {4 hours}.  
For inference, our one-step generator enables the generation of 50k images in only 20 seconds, achieving a  {30$\times$ speedup}.}
    \resizebox{\textwidth}{!}{%
    \begin{tabular}{l|c|ccc|cc}
    \toprule
        \multirow{2}{*}{\textbf{Datasets}} & \multirow{2}{*}{\textbf{Pretraining Time}} & \multicolumn{3}{c}{\textbf{Distillation Time to Achieve the Same FID as}} & \multicolumn{2}{c}{\textbf{Time to Generate 50k Images}} \\ \cmidrule(lr){3-7}
        & & \textbf{Ambient-Full} & \textbf{Ambient-Truncated} & \textbf{Best} & \textbf{Diffusion} & \textbf{DSD} \\ \midrule
        CIFAR-10 & \multirow{4}{*}{\textasciitilde{}2 days} & \ \ \ \ 7 minutes & \ \ \ \textasciitilde{}3 hours & \ \textasciitilde{}3 days & 10 minutes & 20 seconds \\\cmidrule(lr){6-7}
        FFHQ &  & \ \  56 minutes &\ \ \  \textasciitilde{}3 hours & \ \ \textasciitilde{}9 hours & \multirow{3}{*}{15 minutes} & \multirow{3}{*}{30 seconds} \\
        CelebA-HQ & & \ \  34 minutes & \ \ \ \textasciitilde{}2 hours & \textasciitilde{}13 hours & & \\
        AFHQ-v2 &  &\ \ 80 minutes & \ \ \ \textasciitilde{}3 hours & \textasciitilde{}13 hours & & \\
    \bottomrule
    \end{tabular}}
    \label{tab:efficiency}
\end{table}

\begin{figure}[h]
    \centering
\includegraphics[width=1\linewidth]{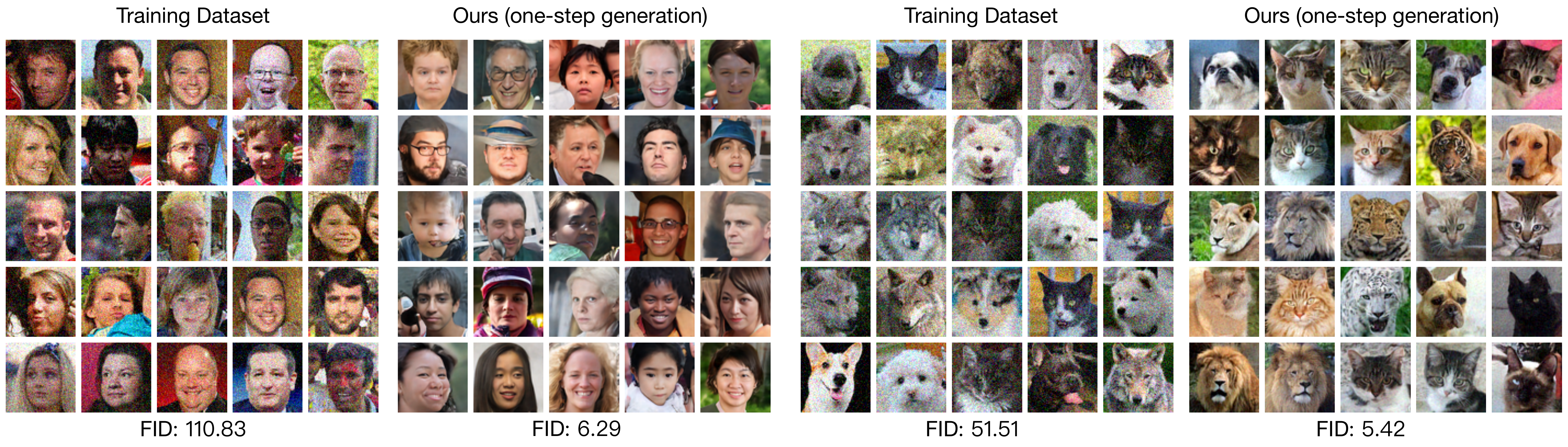}
    \caption{\textbf{Qualitative results of DSD (ours, one-step) at $\sigma=0.2$.}  The left two panels are from FFHQ, while the right two are from AFHQ-v2.  Zoom in for better viewing.}
    \label{fig:qual_results_add}
\end{figure}

\section{Additional Qualitative Results}
\label{app:quality}
In this section, we present additional qualitative results about the noisy training dataset and images generated by our D-SiD model. A quick view is in Fig. \ref{fig:qual_results_add}.

\begin{figure}
    \centering
    \includegraphics[width=\linewidth]{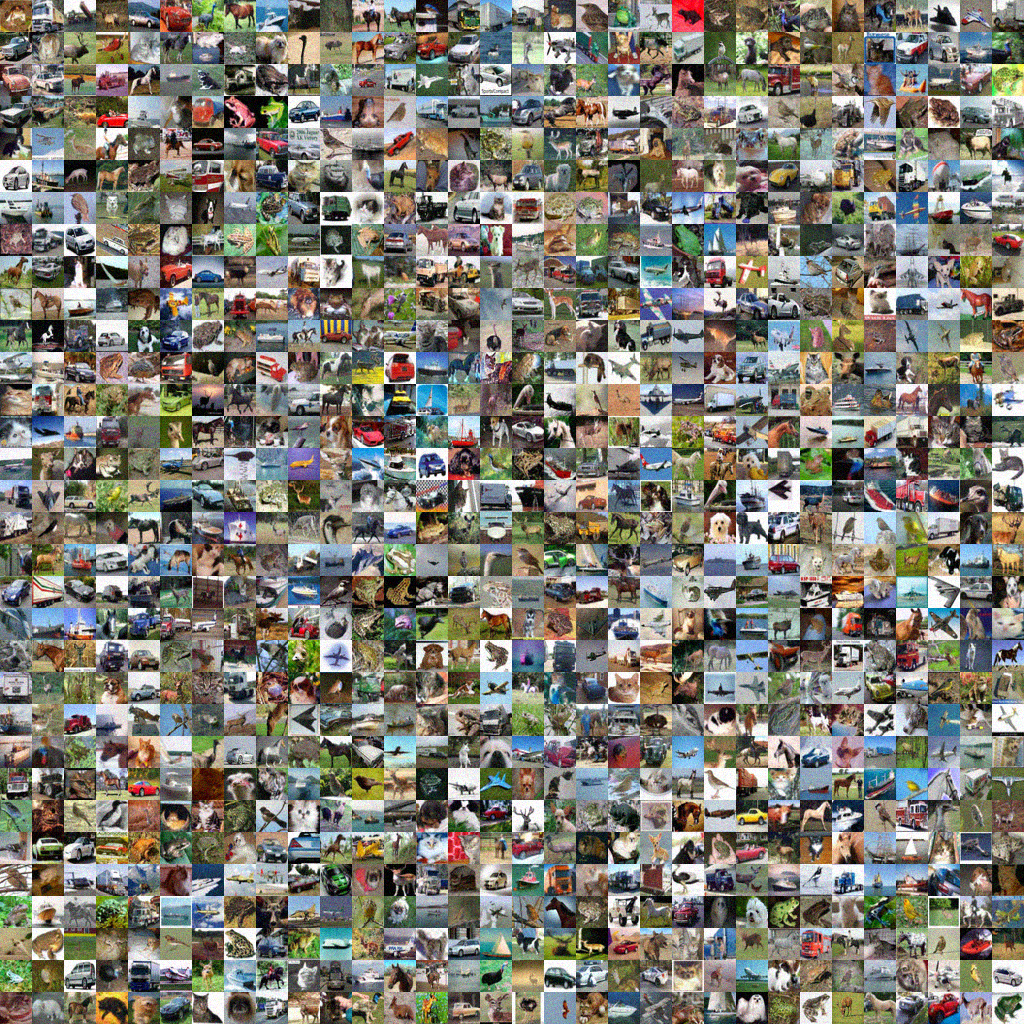}
    \caption{CIFAR-10 32x32 noisy dataset with $\sigma=0.1$ (FID: 73.74).}
    \label{fig:5}
\end{figure}

\begin{figure}
    \centering
    \includegraphics[width=\linewidth]{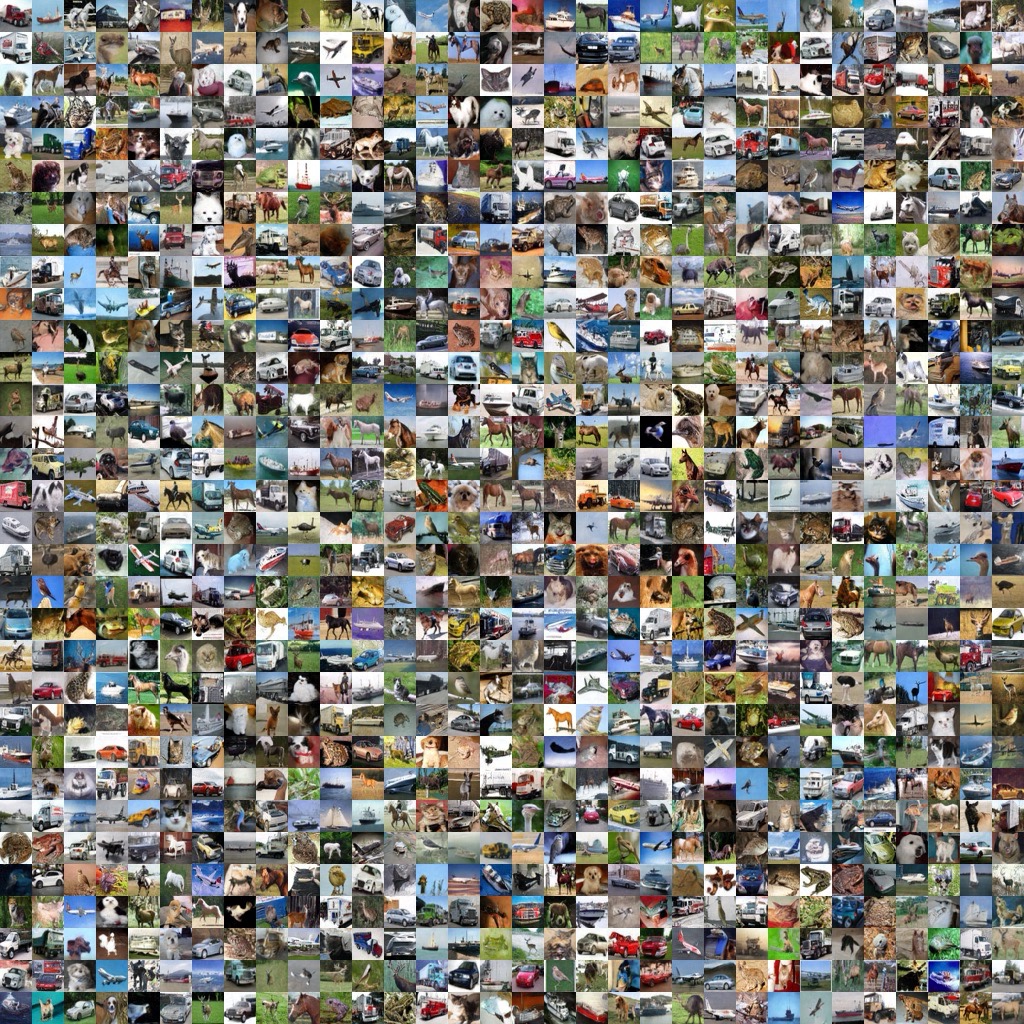}
    \caption{Unconditional CIFAR-10 32x32 random images generated with D-SiD training with noisy dataset with $\sigma=0.1$ (FID: 3.98).}
    \label{fig:4}
\end{figure}

\begin{figure}
    \centering
    \includegraphics[width=\linewidth]{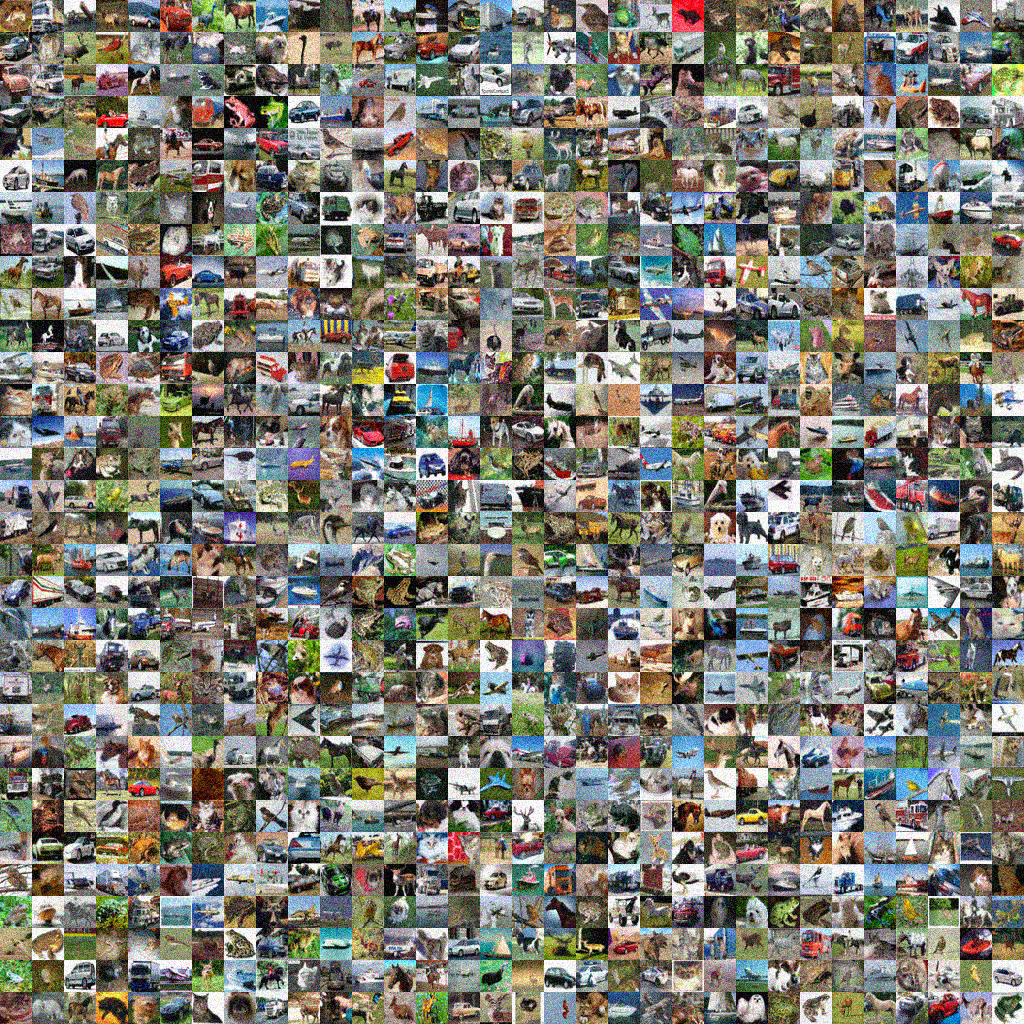}
    \caption{CIFAR-10 32x32 noisy dataset with $\sigma=0.2$ (FID: 127.22).}
    \label{fig:3}
\end{figure}

\begin{figure}
    \centering
    \includegraphics[width=\linewidth]{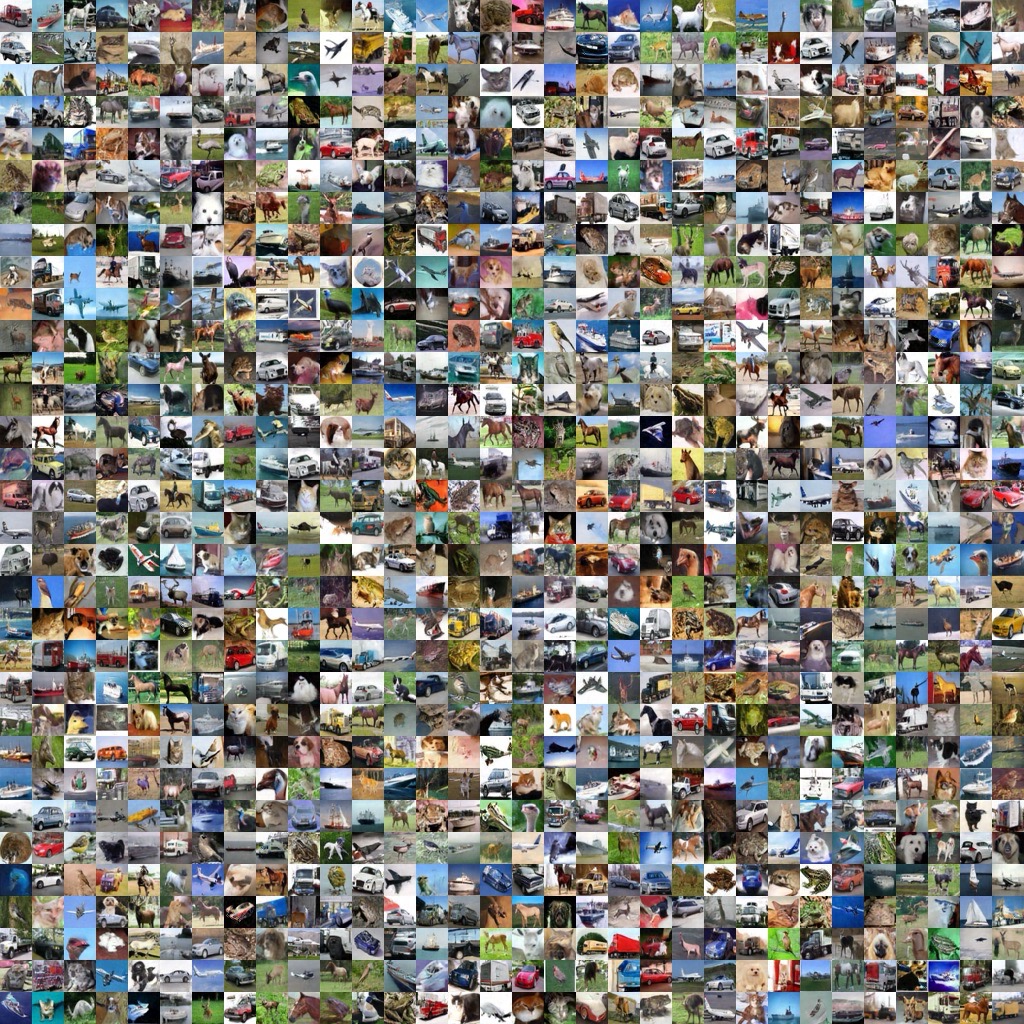}
    \caption{Unconditional CIFAR-10 32x32 random images generated with D-SiD training with noisy dataset with $\sigma=0.2$ (FID: 4.77).}
    \label{fig:2}
\end{figure}

\begin{figure}
    \centering
    \includegraphics[width=\linewidth]{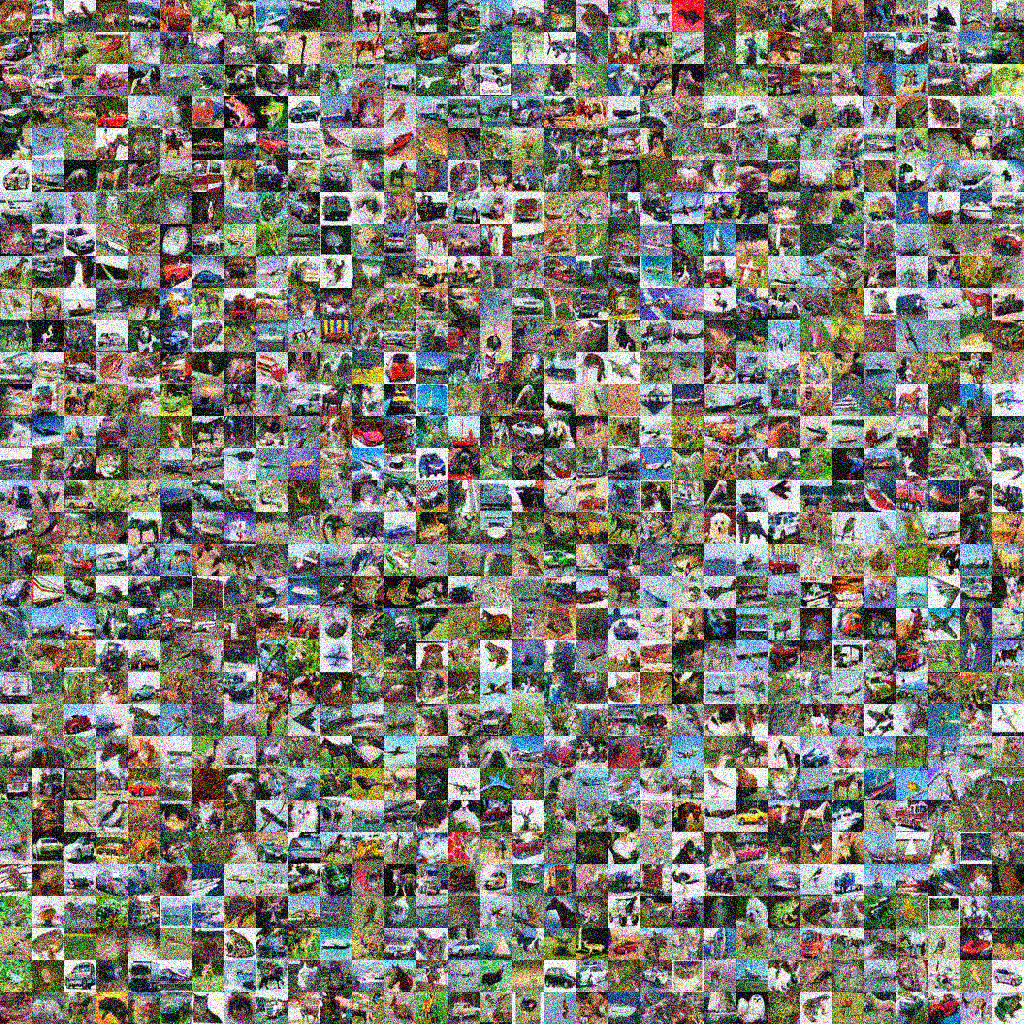}
    \caption{CIFAR-10 32x32 noisy dataset with $\sigma=0.4$ (FID: 205.52).}
    \label{fig:1}
\end{figure}

\begin{figure}
    \centering
    \includegraphics[width=\linewidth]{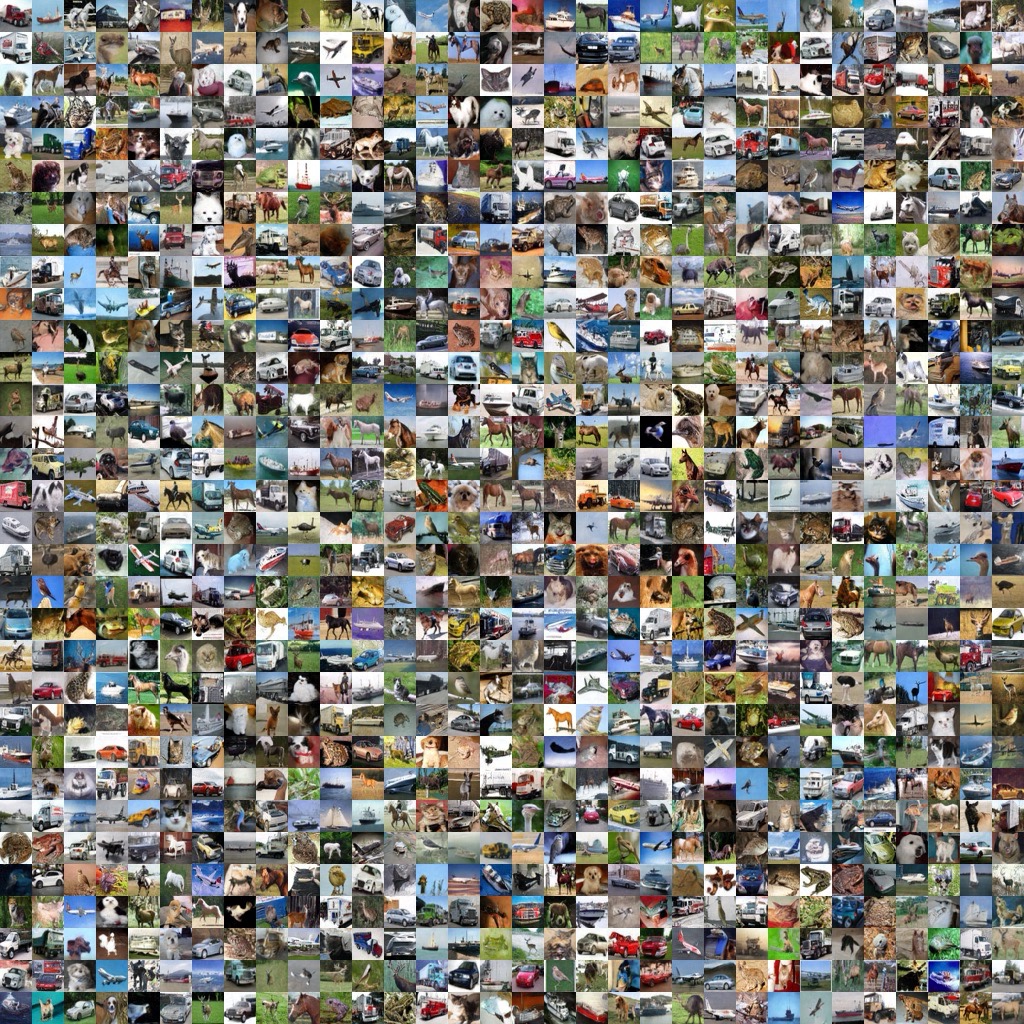}
    \caption{Unconditional CIFAR-10 32x32 random images generated with D-SiD training with noisy dataset with $\sigma=0.4$ (FID: 21.63).}
    \label{fig:cifar-0.4}
\end{figure}


\begin{figure}
    \centering
    \includegraphics[width=\linewidth]{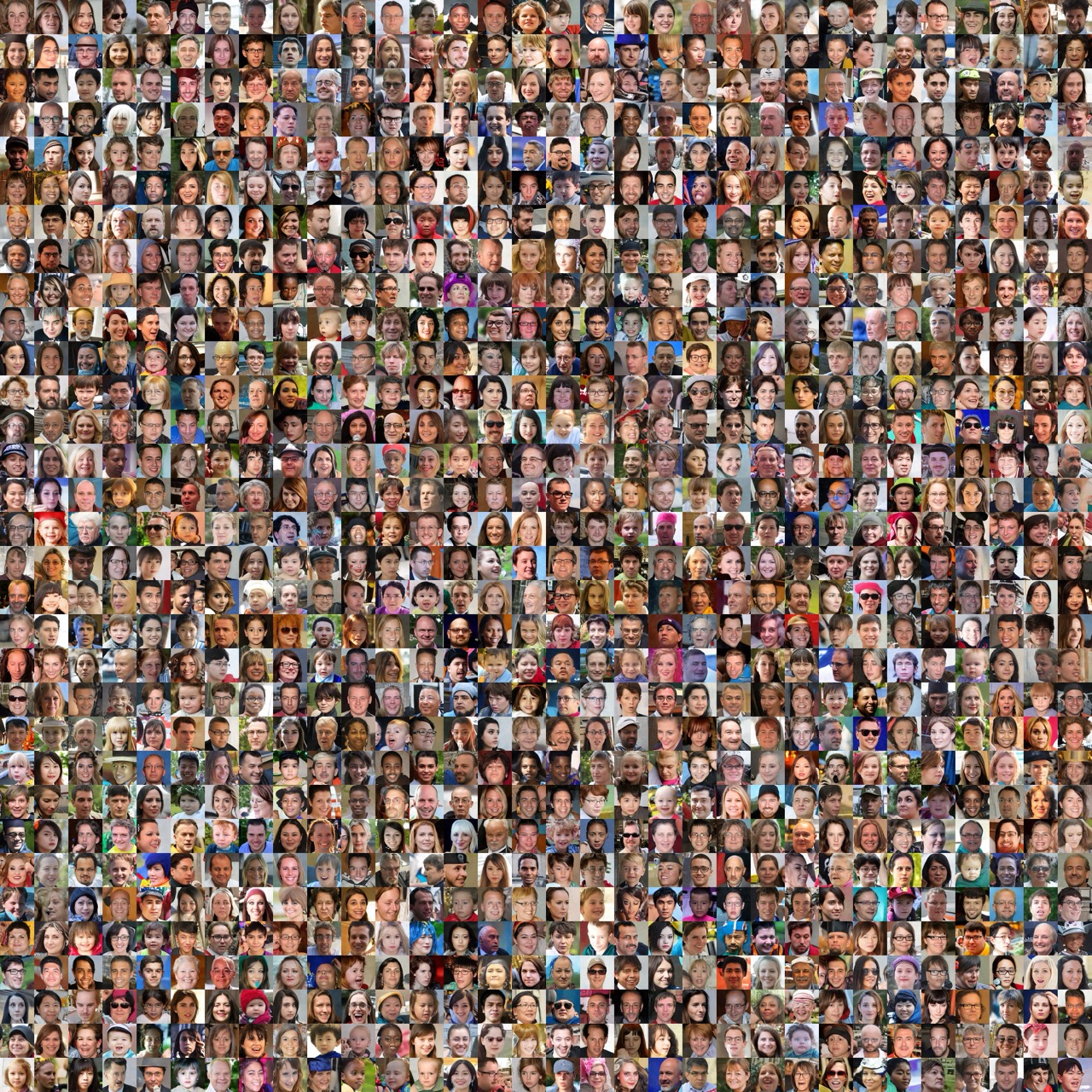}
    \caption{Unconditional FFHQ 64x64 random images generated with D-SiD training on noisy dataset with $\sigma=0.2$ (FID: 6.29).}
    \label{fig:result_ffhq}
\end{figure}


\begin{figure}
    \centering
    \includegraphics[width=\linewidth]{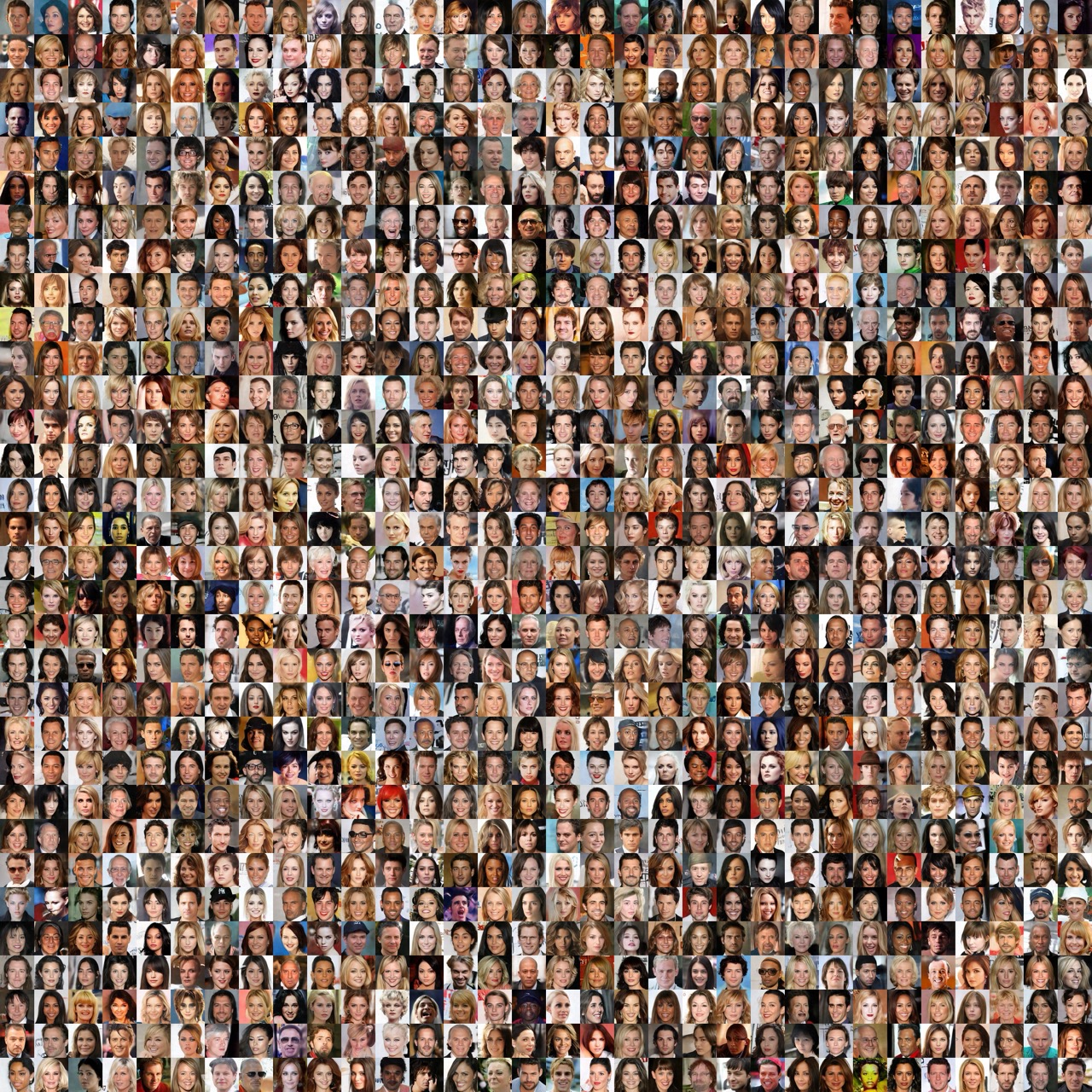}
    \caption{Unconditional CelebA-HQ 64x64 random images generated with D-SiD training on noisy dataset with $\sigma=0.2$ (FID: 6.48).}
    \label{fig:result_celeba}
\end{figure}


\begin{figure}
    \centering
    \includegraphics[width=\linewidth]{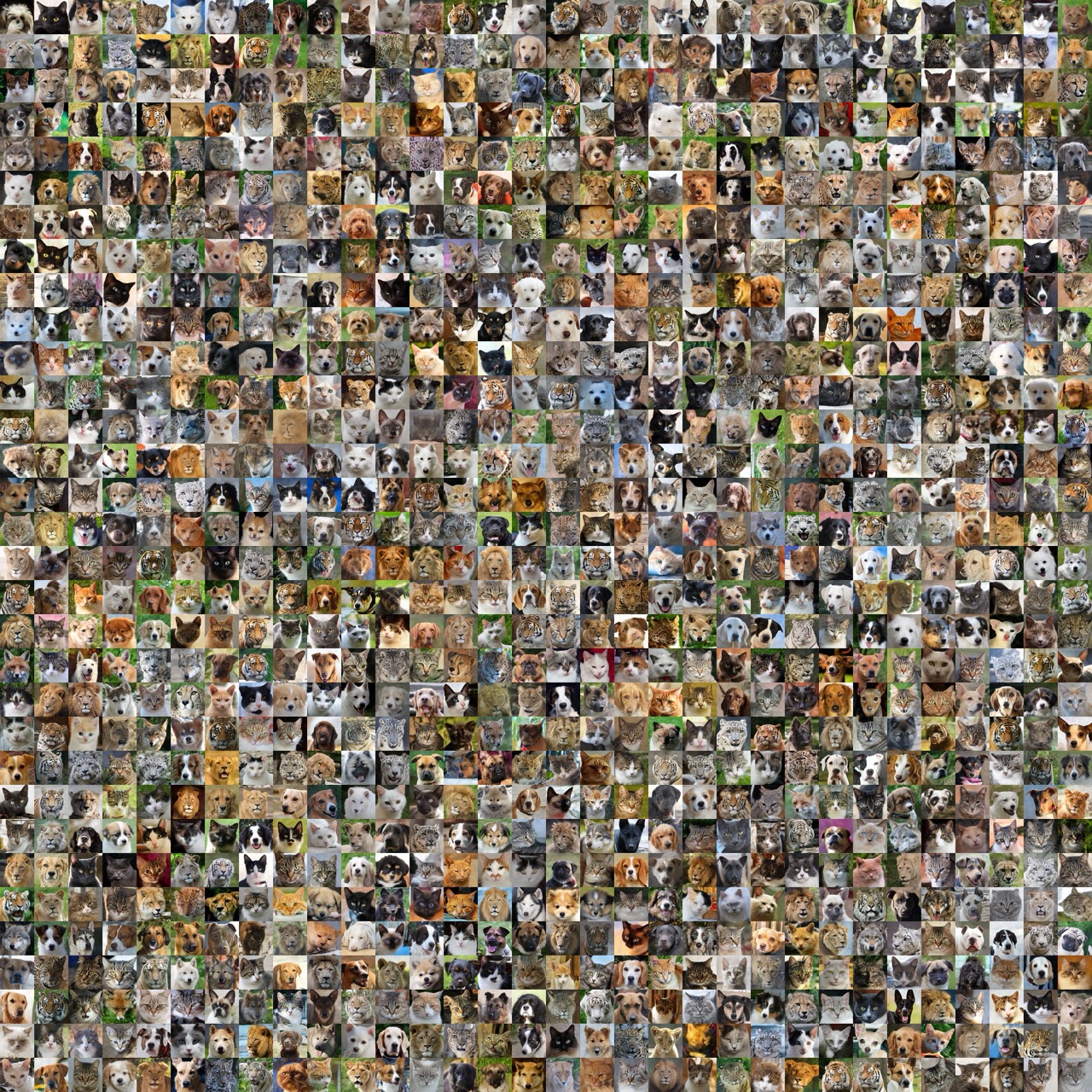}
    \caption{Unconditional AFHQ-v2 64x64 random images generated with D-SiD training on noisy dataset with $\sigma=0.2$ (FID: 5.42).}
    \label{fig:result_afhq}
\end{figure}

\end{document}